\documentclass[lettersize,journal]{IEEEtran}

\usepackage[paperwidth=7.9in,paperheight=10.75in,left=0.5in,right=0.5in,top=0.7in,bottom=0.97in]{geometry}

\usepackage{cite}
\usepackage{amsfonts,latexsym,amssymb}
\usepackage{amsmath,amsthm}
\usepackage{graphicx}
\usepackage[linesnumbered, ruled]{algorithm2e}
\SetAlgoNlRelativeSize{-2}
\SetAlCapSkip{0.5em}
\usepackage{graphics}
\usepackage{enumerate}
\usepackage{epsfig}
\usepackage{cases}
\usepackage{color}
\usepackage{url}
\usepackage{bm}
\usepackage{dsfont}
\usepackage{multirow}
\usepackage{booktabs}
\usepackage{flushend}
\usepackage{graphicx}
\usepackage{subcaption}
\usepackage{caption} 
\usepackage{mathrsfs}
\usepackage{makecell}
\usepackage[colorinlistoftodos,prependcaption,textsize=tiny]{todonotes}

\newtheorem{proposition}{Proposition}
\newtheorem{theorem}{Theorem}
\theoremstyle{definition}
\newtheorem{definition}{Definition}

\renewcommand\color[1]{\relax}
\renewcommand\textcolor[2]{#2}
\color{black}

\def\BibTeX{{\rm B\kern-.05em{\sc i\kern-.025em b}\kern-.08em
    T\kern-.1667em\lower.7ex\hbox{E}\kern-.125emX}}

\IEEEoverridecommandlockouts
\begin{document}
\title{\huge 
Joint Optimization of Training and Inference in Federated Edge Learning via Constrained Multi-Objective Deep Reinforcement Learning
}

\author{Zhen Li,
        Jun Cai,~\textit{Senior Member, IEEE},
        Chao Yang,
        and
        Haoran Gao

\IEEEcompsocitemizethanks{
\IEEEcompsocthanksitem Zhen Li, Jun Cai (Corresponding author), and Haoran Gao are with the Department of Electrical and Computer Engineering, Concordia University, Montreal, QC, H3G 1M8, Canada. (E-mail: zhen\_li@ieee.org, jun.cai@concordia.ca, haoran.gao@mail.concordia.ca).
\IEEEcompsocthanksitem Chao Yang is with the School of Automation, Guangdong University of Technology, China. (E-mail: chyang513@gdut.edu.cn).
}}


\IEEEtitleabstractindextext{%
\begin{abstract}
Federated edge learning (FEEL) has recently emerged as a promising paradigm for achieving edge intelligence (EI) via enabling collaborative model training across edge devices while protecting data privacy.
In this paper, we put forth an online optimization framework that jointly manages federated training and inference on resource-constrained edge devices.
We introduce a tandem-queue-inspired conversion mechanism that bridges inference requests and training data, and further incorporate both data and model freshness into the accuracy formulation to capture temporal dynamics in real-world environments.
To maximize inference accuracy while minimizing latency and energy consumption, the mode selections, communication, and computation resource allocations of edge devices are jointly optimized.
We formulate this optimization as a multi-objective optimization problem, which is NP-hard and
further complicated by the online setting.
To address these challenges, we transform the problem into a multi-objective Markov decision process (MOMDP) and develop a \underline{c}onstrained \underline{m}ulti-\underline{o}bjective \underline{p}roximal \underline{p}olicy \underline{o}ptimization (C-MOPPO) algorithm.
Specifically, C-MOPPO first learns a set of policies with different preferences across three objectives, then leverages constrained policy optimization to enrich the Pareto front and obtain high-quality, dense solutions.
Extensive experiments demonstrate that C-MOPPO achieves well-balanced trade-offs among objectives and significantly outperforms
baselines under various system configurations.
\end{abstract}

\begin{IEEEkeywords}Federated edge learning, inference, resource allocation, multi-objective optimization, deep reinforcement learning.
\end{IEEEkeywords} }
\maketitle

\IEEEdisplaynontitleabstractindextext

\section{Introduction}
\IEEEPARstart{W}{ith}  the proliferation of privacy-preserving artificial intelligence (AI) applications, federated learning (FL) has emerged as a promising paradigm that enables collaborative model training across clients without exposing their raw data. 
Owing to this advantage, FL has attracted growing attention from both academia and industry, with applications spanning from the Internet of Things (IoT) to emerging generative AI (GenAI)~\cite{NguyenD, HuangX}.
Nevertheless, conventional FL frameworks rely heavily on centralized servers, suffering from high communication latency and limited scalability in delay-sensitive applications.
To address these challenges, the integration of FL with edge computing paradigms, known as federated edge learning (FEEL), has been widely adopted as a means to leverage the computational capabilities of edge devices to process data closer to the source~\cite{TaoM}.
Unlike traditional FL that focuses primarily on collaborative training, FEEL explicitly targets network edge deployment with distributed edge devices as clients.
This framework enables faster response, improved privacy, and reduced backbone network load, making it suitable for latency-sensitive and privacy-critical applications~\cite{LimW, DuanQ}.

In FEEL systems, edge devices, such as autonomous vehicles, smartphones, and IoT sensors, ordinarily operate under limited computational capacity, restricted communication bandwidth, and finite battery life.
Although recent studies~\cite{RenJ, SunY, CaoX} have overcome some constraints and improved learning efficiency, critical issues closely related to practical deployments remain insufficiently explored.
On the one hand, most existing works treat federated training and inference as separate processes, overlooking their coexistence and interaction in FEEL systems.
This coexistence contrasts with conventional FL, where clients first participate in training to build the global model, then perform inference once the fully-trained model is deployed.
However, edge devices in FEEL systems are required to participate in FL while also serving real-time inference requests from end users~\cite{HanP, LuoK}.
For example, in intelligent transportation systems (ITS), roadside sensing nodes such as cameras may collaboratively train neural network models via FL to improve object detection accuracy, while using current models to infer the locations of vehicles and pedestrians.
In this context, the mode of an edge device in each round corresponds to its working state: either training or inference. 
Training involves updating model parameters through FL, and inference involves processing requests with fixed model parameters.
Due to this fundamental difference, an edge device cannot perform both processes simultaneously within the same temporal window.
This necessitates the development of efficient mode selection mechanisms that can dynamically determine, for each round, whether an edge device should participate in federated training or focus on serving inference requests.

On the other hand, when training and inference are simultaneously considered, the static dataset used in the traditional FL framework becomes ineffective.
As inference request patterns evolve, models trained on static datasets will fail to capture such dynamics, necessitating continuous model retraining with up-to-date data to ensure high-quality services~\cite{WangH, AiX}.
For instance, the nodes in the ITS must update their models with newly collected data to capture evolving traffic situations.
This calls for a novel modeling framework and efficient resource management mechanisms to meet quality of service (QoS) requirements, such as inference accuracy and latency, and reduce the energy costs of edge devices at the same time.

However, addressing the above features in FEEL systems raises new challenges. Specifically,

\begin{enumerate}
    \item It is inherently challenging to determine the mode between training and inference, as both consume time and energy resources, while impacting performance metrics such as inference latency and accuracy differently. The resulting mode selection problem is complex and difficult to analyze.
    Additionally, the well-known \textit{straggler problem} in federated training further complicates the decision-making, since each device's model performance depends on both its own choices and the choices of others~\cite{ChenX}.
	
    \item To bridge the evolving request patterns and dynamic data, temporal correlations need to be considered, 
    rendering inference accuracy time-sensitive and more complex.
    Specifically, it must account for not only the volume of data but also data and model freshness~\cite{Shisher, ZhangY}.
    Furthermore, given that edge devices cannot access future information, 
    an online decision-making framework becomes necessary, 
    introducing significant uncertainty and complexity to the joint optimization problem.
	
    \item The designed strategy should account for multiple performance metrics, including inference accuracy, latency, and energy consumption. 
    These objectives are often conflicting with each other in practical systems, making it hard to optimize them simultaneously~\cite{ZhouY}. 
    This results in a multi-objective optimization (MOO) framework, posing substantial difficulties to achieve effective solutions.
\end{enumerate}

To address the aforementioned challenges, in this paper, we present a novel joint optimization framework that guides resource-constrained edge devices in the FEEL system to make mode selection and resource allocation decisions properly. 
Specifically, we consider the coexistence of federated training and inference, where edge devices dynamically balance these two processes
under limited communication and computation budgets. 
Given that edge devices need to make sequential decisions under dynamic and uncertain environments while optimizing conflicting objectives such as inference accuracy, latency, and energy consumption, we formulate the problem as a multi-objective Markov decision process (MOMDP). 
To this end, with the NP-hard nature of the formulated MOMDP, we develop a constrained multi-objective proximal policy optimization (C-MOPPO) algorithm that learns efficient policies in an online manner without requiring future information and effectively explores the trade-offs among conflicting performance metrics.
The main contributions of this paper are summarized as follows.
\begin{itemize}
	\item We investigate a joint optimization of mode selection, CPU frequency, and transmission power allocation for resource-constrained edge devices in the FEEL system. 
    A tandem-queue-inspired conversion mechanism is designed to characterize the transition from inference requests to training data, provisioning adaptive dynamic datasets for edge devices.
    \textcolor{blue}{
    To capture the temporal relevance of training data and the staleness of local models, we integrate the concepts of 
    the age of data (AoD) and age of model (AoM) to 
    develop an age-sensitive inference accuracy formulation that jointly accounts for the volume and aging of both 
    training data
    and edge models.}
	
	\item We introduce a novel algorithm, called C-MOPPO, to address the formulated multi-objective optimization problem. 
    Specifically, C-MOPPO first leverages proximal policy optimization (PPO) under different preferences to generate a set of candidate policies.
    Based on these candidates, it then alternatively performs policy selection and constrained policy optimization, extending the policies to achieve balanced trade-offs among inference accuracy, latency, and energy consumption.
	
	\item We conduct extensive experiments to verify the proposed C-MOPPO algorithm. 
    The results validate its feasibility in generating a high-quality and dense Pareto set and demonstrate its consistent superiority over benchmark approaches under various system settings.  
\end{itemize}

The rest of this paper is structured as follows.
Section~\ref{rw} reviews recent related works and outlines the novelties of this paper.
Section~\ref{syst} details the system model and formulates the multi-objective optimization problem. 
In Section~\ref{algo}, the C-MOPPO is presented and analyzed theoretically.
Section~\ref{resu} reports the experiment and simulation results, and Section~\ref{concl} concludes the paper.


\section{Related Work}\label{rw}
Recently, numerous studies have been devoted to FEEL systems, which have emerged as a key paradigm for enabling EI.
For example, 
Tang et al.~\cite{TangJ} improved training efficiency by proposing a class-balanced client selection and bandwidth allocation framework, aiming to reduce both the latency and energy consumption of federated training in mobile edge computing networks.
Fu et al. in~\cite{FuS} investigated the execution of multiple federations in FEEL, and formulated the problem as a two-stage Stackelberg game, where resource allocation was optimized through convex modeling and device selection was treated as a congestion game.
In~\cite{FengJ}, the authors formulated a joint optimization problem to simultaneously minimize the energy consumption of heterogeneous edge devices and maximize their harvested energy, which was then solved through a decomposition-based approach. 
These works primarily concentrated on the federated training process, while overlooking a critical aspect that edge devices are also responsible for handling inference requests with their trained models. 

Some recent works have been dedicated to investigating edge inference for supporting low-latency service delivery. Specifically,
Fan et al. in~\cite{FanW} leveraged a Lyapunov optimization framework to jointly optimize inference task offloading, adjustment, and resource allocation, with the objective of minimizing inference delay while satisfying accuracy requirements across edge devices. 
To improve inference efficiency in unmanned aerial vehicles assisted IoT networks, Xu et al.~\cite{XuJinf} proposed a constrained deep reinforcement learning based algorithm for task allocation and resource management.
Nevertheless, these studies assume that model performance remains constant over time, which limits their applicability to practical online EI systems, as the inference accuracy often fluctuates due to the temporal drift in request patterns.
More recently, several studies have attempted to integrate FL with inference. For instance, 
Luo et al.~\cite{LuoK} introduced a proactive optimization approach to balance the trade-off between model performance and the cost of federated training under the inference offloading framework.
Lackinger et al. in~\cite{Lackinger} proposed a load-aware inference orchestration scheme for hierarchical federated learning to reduce the latency and communication costs. 
However, these works still face limitations, such as implicitly assuming that clients can perform training and inference simultaneously, which is unrealistic for resource-constrained edge devices. 
Moreover, existing studies rely on inference offloading to enhance service performance, which may introduce potential privacy concerns and require additional encryption mechanisms.

Deep reinforcement learning (DRL) algorithms have been widely applied in federated learning to tackle online optimization problems and improve long-term performance in dynamic environments, such as client selection, resource allocation, and cost minimization~\cite{ChenX, QiangG, LiZ, OkegbileS}.
However, in these papers, the authors typically formulated the problems as single-objective optimizations by summarizing multiple metrics with predefined weights. In practice, performance objectives can be conflicting, making standard DRL algorithms less effective in capturing the comprehensive trade-offs. This limitation highlights the need for approaches that can better handle the conflicting objectives in FEEL systems.

In summary, unlike existing works, this paper simultaneously considers training and inference in the FEEL system and proposes a novel \textcolor{blue}{multi-objective deep reinforcement learning (MODRL)} algorithm that jointly optimizes mode selection, communication, and computation resource allocations for resource-constrained edge devices under dynamic accuracy modeling within the online settings.


\section{System Model and Problem Formulation}\label{syst}
In this section, we first present an overview of the considered FEEL system in Section~\ref{overview}.
Then, to be more specific, the detailed analysis for federated training and request processing, as well as the modeling of inference accuracy, are described in Section~\ref{mode} and~\ref{acc}, respectively.
Finally, 
a corresponding multi-objective optimization problem is formulated in Section~\ref{formulation}.
To facilitate understanding, Table~\ref{notation} presents all important notations used in this paper.
\begin{table}[t]
	\begin{center}
		\caption{Summary of Important Notations}
		\label{notation}
		\resizebox{\linewidth}{!}{ 
				\begin{tabular}{c|p{0.725\linewidth}}
					\hline
					\textbf{Symbol} &  \textbf{Meaning} \\\hline \hline
					$\mathcal{N}$ &  Set of edge devices \\
					$\mathcal{M}^t$ & Set of edge devices in training mode in round $t$ \\
					$\lambda_n$ & Request arrival rate of edge device $n$\\
					$\alpha^t_n$ &  Mode selection of edge device $n$ in round $t$ \\
					$f_n^t$ & CPU frequency of edge device $n$ in round $t$ \\
					$p_n^t$ & Power allocation of edge device $n$ in round $t$ \\
					$L^t_{n,\text{cmp}}$, $E^t_{n,\text{cmp}}$ & Latency, energy consumption associated with computing model update for edge device $n$ in round $t$\\
					$L^t_{n,\text{up}}$, $E^t_{n,\text{up}}$ & Latency, energy consumption associated with uploading model update from edge device $n$ in round $t$\\
					$L^t_{n,\text{tr}}$, $E^t_{n,\text{tr}}$ & Total latency, total energy consumption of edge device $n$ in training mode in round $t$\\
					$L^t_{n,\text{in}}$, $E^t_{n,\text{in}}$ & Latency, energy consumption of edge device $n$ in inference mode in round $t$\\
					$L^t_{n,\text{wa}}$ & Average waiting time for requests of edge device $n$ in round $t$\\
					$L^t$ & Total duration of round $t$\\
					$Q_n^t$ & Number of request queue of edge device $n$ in round $t$\\
					$D_n^t$ & Number of training data of edge device $n$ in round $t$\\
					$\Delta_n^t$ & Average age of data for edge device $n$ in round $t$\\
					$\delta_n^t$ & Age of model for edge device $n$ in round $t$\\
					\hline
				\end{tabular}
			} 
		\end{center}
		\vspace{-0.2in}
	\end{table}

\subsection{System Overview}\label{overview}
As illustrated in Fig.~\ref{system}, the FEEL system comprises one edge server (ES) and a set of resource-constrained edge devices, denoted by $\mathcal{N} = \{1,2,...,N\}$.
The edge devices are required to simultaneously provide real-time inference services and update models through federated training to maintain the inference accuracy. 
However, due to mode conflicts and resource limitations, we assume that the edge devices can only perform either training or inference during each round, making mode selections essential.
The detailed workflow of the system is described as follows.
\begin{figure}[t]
	\centering
	\includegraphics[width=1\linewidth]{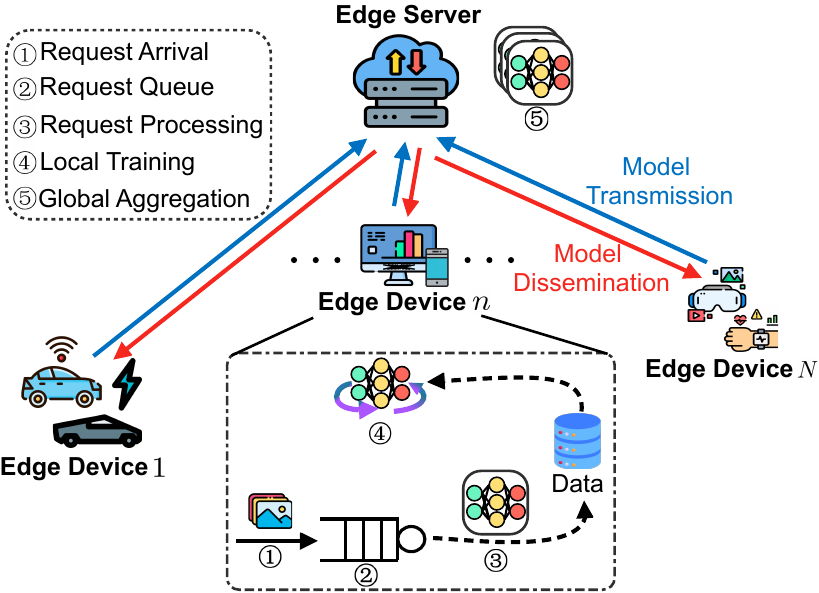}
	\caption{Illustration of the FEEL system.}
	\label{system}
	\vspace{-0.1in}
\end{figure}

\subsubsection{Model Initialization}
Each edge device is initially configured with a pre-trained learning model.
\subsubsection{Request Arrival}
During each round $t\in \mathcal{T} = \{1,2,\dots,T\}$, edge devices receive complex inference requests.
These inference tasks are first pushed into a request queue, awaiting processing.
\subsubsection{Edge Devices Perform Local Training or Request Processing}
In each round $t$, the mode of edge device $n$ is denoted by $\alpha_{n}^t$, with $\alpha_{n}^t=1$ if device $n$ participates in the $t$-th round federated training (\textbf{Training Mode}), and $\alpha_{n}^t=0$ if device $n$ processes the inference requests waiting in its queue (\textbf{Inference Mode}). 
For the edge devices in training mode, they transmit their model updates to the ES after completing local training. 
In contrast, edge devices in inference mode process requests using the stored model 
and return the results to users for evaluation. 
The completed requests, along with user feedback, are then converted into labeled training samples that are temporarily stored and utilized in the subsequent training~\cite{SharmaR}.
After training, these samples are discarded, thereby ensuring the storage constraints in practical deployments.
This design enables edge devices to dynamically accumulate training data without relying on large storage, thus improving training efficiency under resource constraints.
\subsubsection{Model Aggregation and Dissemination}
After receiving the model updates from all participating edge devices, the ES aggregates them using a weighted averaging mechanism (e.g., FedAvg in~\cite{McMahan}).
The updated global model is then distributed to the participating devices with $\alpha_{n}^t=1$ as their training rewards.

\subsection{Federated Training and Request Processing}\label{mode}
In what follows, we analyze the detailed latency and energy consumption in the training and inference modes.
\textcolor{blue}{In our system model, we consider a round-level service, assuming that the resource allocation remains constant for the entire duration of a round once determined at its beginning.}

$\bullet$ \textit{Federated Training.}
Let $\mathcal{M}^t=\{n|\alpha_{n}^t=1, n\in\mathcal{N}\}$ be the subset of edge devices that are in the training mode in round $t$.
These edge devices first receive the latest global model from the ES.
Subsequently, the edge devices load all the new training data from their local cache and perform $U$ iterations of full-batch gradient descent.
Denote the new data stored in edge device $n$'s cache in the $t$-th round by $\mathcal{D}^t_n$ with the number $D^t_n$.
The computation latency  for edge device $n$ to participate in the current training round can be calculated as
\begin{equation}
L_{n,\text{cmp}}^t =U\frac{\chi_{n,\text{tr}} D^t_n}{f^t_{n}},
\end{equation}
where 
$\chi_{n,\text{tr}}$ is the required CPU cycles for calculating the gradient on a single data sample at edge device $n$, and $f^{t}_{n} $ represents the CPU frequency allocated to the $n$-th edge device during the round $t$\textcolor{blue}{\footnote{\textcolor{blue}{While formulated using CPU terminology, our framework actually treats $\chi$ and $f$ as abstract representations of computational demand and resource capacity, respectively.
For GPU-equipped edge devices, these parameters can correspond to GPU processing cycles and core frequency.}}}.
According to~\cite{YangZ}, the energy consumption for computing the model update can be expressed as
\begin{equation}
    E_{n,\text{cmp}}^t =e_n(f^t_{n})^3L_{n,\text{cmp}}^t,
\end{equation}
where $e_n$ is the effective switched capacitance of edge device $n$, which depends on its chipset architecture.

After completing local gradient computation, the participating edge devices upload their model updates to the ES for aggregation.
Let $p_{n}^t$ be the transmission power allocated to edge device $n$ in round $t$.
The transmission rate from edge device $n$ to the ES is
\begin{equation}
    \mathsf{r}_{n}^t = \frac{B}{|\mathcal{M}^t|}\log_2(1+\frac{p_{n}^t h^t_{n}}{\sigma^2}),
\end{equation}
where $B$ is the bandwidth of the channel that is equally shared by all participating edge devices.
$h^t_n$ and $\sigma^2$ are the channel gain and channel noise, respectively.
Correspondingly, the latency and energy consumption of the edge device $n$ in uploading the model update with the size $G$
can be calculated as
\begin{align}
    L^{t}_{n,\text{up}} &= \frac{G}{\mathsf{r}^t_n}, \\
    E^{t}_{n,\text{up}} &= p_n^t L^t_{n,\text{up}}.
\end{align}

Considering the synchronous FL framework, the total latency of the $t$-th round of training, which is also defined as the total duration of the $t$-th round, depends on the slowest edge device and can be given as
\begin{equation}
     L^t_{n,\text{tr}} =  L^t = \max_{n\in \mathcal{M}^t}(L_{n,\text{cmp}}^t+L^{t}_{n,\text{up}}).
\end{equation}
\textcolor{blue}{Given that the edge server possesses significantly higher transmit power than edge devices and utilizes the wider downlink bandwidth for broadcasting, the latency and energy consumption associated with model dissemination are considered negligible compared to the uplink and computation costs\cite{TangJ,ShiW}. Therefore,}
the total energy consumption of the participating edge device $n$ in the $t$-th round can be derived as
\begin{equation}
     E^t_{n,\text{tr}} =E_{n,\text{cmp}}^t + E^t_{n,\text{up}}.
\end{equation}

$\bullet$ \textit{Request Processing.}
For the edge devices selected to be in inference mode during the $t$-th round, their learned models are frozen for request processing.
\textcolor{blue}{We adopt a batch-processing paradigm to enhance the utilization of computation resources. 
Specifically, all inference requests accumulated in the edge device’s queue are batched into a single large input tensor, which is then fed into the model for parallel inference~\cite{LiuZ}. 
This setting avoids the sequential processing mechanism, which may underutilize the computing units of edge devices.}
Let $\lambda_n$ be the request arrival rate of edge device $n$, the number of requests generated
during the round $t$ is given by $d^t_n = \lambda_n L^t$, where $L^t$ is the total duration of round $t$.
Denote the request queue in edge device $n$'s cache in round $t$ by $\mathcal{Q}^t_n$ with the size $Q^t_n$, which is updated according to
\begin{equation}
    Q^{t+1}_n =
\begin{cases}
d^t_n,\ \alpha_n^t=0,\\
Q^{t}_n + d^t_n,\ \alpha_n^t=1.
\end{cases}
\end{equation}

According to the CPU frequency allocation $f^t_n$, the processing time and energy consumption for edge device $n$ to infer $Q^t_n$ requests can be respectively expressed as
\begin{align}
    L^t_{n,\text{in}} &= \frac{\chi_{n,\text{in}} Q^t_n}{f^t_{n}}, \\
    E^t_{n,\text{in}} &= e_n (f^t_{n})^3L^t_{n,\text{in}},
\end{align}
where $\chi_{n,\text{in}}$ is the CPU cycles needed to process an inference request on edge device $n$.
Given that inference services are subject to much stricter latency requirements than training in practical systems, we assume that the inference latency must satisfy the following constraint
\begin{equation}
    L_{n,\text{in}}^t \le L^t,\ \forall n\in\mathcal{N} \setminus \mathcal{M}^t,\ \forall t\in\mathcal{T}.
\end{equation}
Based on the number of requests generated during each round, 
the average waiting time for these requests until edge device $n$ is in inference mode 
can be calculated as
\begin{equation}
L^t_{n,\text{wa}} = \frac{1}{Q^t_n} \sum_{t'=t^*}^{t-1} d^{t'}_n \cdot \left( \sum_{t''=t'}^{t-1} L^{t''} \right),
\end{equation}
where $t^* = \max\{\tau < t \mid \alpha_n^\tau = 0\}$ represents the most recent round before $t$ when device $n$ was in the inference mode.
Therefore, the total latency of requests inferred by edge device $n$ in the $t$-th round can be represented as
\begin{equation}
L^t_{n,\text{tot}} = L^t_{n,\text{wa}} + L^t_{n,\text{in}}.
\end{equation}

\subsection{Inference Accuracy}\label{acc}
Building on the conversion mechanism that bridges the request queue and training dataset, 
the amount of the new training data for the edge device $n$ is updated as
\begin{equation}
D^{t+1}_{n} =
\begin{cases}
D_n^{t}+\mathscr{F}_n(Q_n^{t}),\ \alpha^t_n=0,\\
0,\ \alpha^t_n=1,
\end{cases}
\end{equation}
where $\mathscr{F}_n(\cdot)$ is a task-dependent mapping function that converts the number of completed inference requests into the number of new training data\footnote{Note that this mapping varies across applications. For instance, in image classification, each request may yield one training sample (one-to-one mapping); in video-based detection, several frames may be merged into one sample (many-to-one mapping); and in health monitoring, one request may generate multiple samples corresponding to different physiological indicators (one-to-many mapping).}.

\textcolor{blue}{Due to the dynamic nature of online inference, either calculating or transmitting the true model accuracy after each round is infeasible. The former is complex and impractical~\cite{ZhouY}, while the latter incurs communication overheads and poses privacy leakage risks to the FEEL system~\cite{ChangH}. Therefore, rather than relying on the true inference accuracy, we adopt a novel approach that combines the volume and freshness of training data provided by edge devices to estimate the inference accuracy.}
Specifically, the AoD, serving as a metric for quantifying data freshness, is used to capture the temporal relevance of training samples, reflecting how well these data represent recent request patterns. 
Denoted the average AoD at edge device $n$ in round $t$ by $\Delta_n^{t}$, which is updated as\textcolor{blue}{~\cite{ZhangX, HanP, LiT}}
\begin{equation}
    \Delta_n^{t+1}=
    \begin{cases}
        \frac{(\Delta_n^t+L^t)D_n^t + \mathscr{F}_n(Q_n^t)}{D^{t+1}_n} , \ \alpha_n^t=0,
        \\
        1, \ \alpha_n^t=1.
    \end{cases}
\end{equation}
\textcolor{blue}{Note that resetting $\Delta_n^{t+1}$ to $1$ upon successful training ($\alpha_n^t=1$) reflects that the local buffer of the edge device now consists entirely of new data samples, ensuring physical and mathematical consistency across two modes.}
Meanwhile, the volume of training data is employed to reflect the coverage and diversity of samples, contributing to model generalization and stability.
By incorporating both AoD and data volume,  the estimated global model accuracy on the ES side after the $t$-th round is given by
\begin{equation}
    Acc^t_{\text{ser}} = \mathscr{G}\big(\sum_{n\in \mathcal{M}^t}D^t_n,\frac{\sum_{n\in \mathcal{M}^t}\Delta_n^t}{|\mathcal{M}^t|}\big),
\end{equation}
where $\mathscr{G}(\cdot, \cdot)$ denotes an accuracy mapping function.
\textcolor{blue}{Here, the construction of $\mathscr{G}(\cdot, \cdot)$ is motivated by two widely observed empirical phenomena in modern machine learning: scaling law and concept drift.
According to~\cite{Hestness, ZhangY2}, a model's generalization capability scales predictably (often following a logarithmic or power-law curve) with the volume of training data, exhibiting diminishing marginal returns. 
On the other hand, relying on outdated data (i.e., high AoD) severely degrades the model's predictive performance, a staleness effect typically formulated as an exponential decay~\cite{Arafa, AiX2}.
Therefore, we formulate the default accuracy mapping function in our system as a logarithmic-exponential model: $\mathscr{G}(x, y) = \kappa_1 \log(1 + \kappa_2 x) \cdot e^{-\kappa_3 y}$, where $\kappa_1$, $\kappa_2$, and $\kappa_3$ are constants to bound the modeled accuracy within the valid range. 
}

Similar to the modeling approach in~\cite{WuB}, we introduce AoM to quantify the staleness of the local model at edge device $n$ after federated training, calculated by
$\delta_n^{t+1} = (\delta_{n}^{t} + L^t)(1-\alpha_n^t)$.
\textcolor{blue}{Based on the AoM, we assume that the modeled inference accuracy of the local model at edge device $n$ in the $t$-th round is given by}
\begin{equation}
    Acc_n^t = e^{-\varrho \delta_n^t}Acc_{\text{ser}}^{\tilde{t}},
\end{equation}
where $\tilde{t} = \max \{ \tau \leq t \mid \alpha_n^\tau = 1 \}$ is the most recent round $\tau$ in which device $n$ synchronized its local model with the ES (i.e., participated in FL), and $Acc_{\text{ser}}^{\tilde{t}}$ is the server-side global model accuracy after that round.
Note that the decay term $e^{-\varrho\delta_n^t}$ reflects the impact of model staleness on inference accuracy, where $\varrho > 0$ controls the sensitivity.

\subsection{Problem Formulation}\label{formulation}
To evaluate the performance of edge devices in the presented system, we consider three key metrics: inference accuracy, latency, and energy consumption.
Specifically, inference accuracy and latency are two primary indicators of QoS experienced by end users, determining the real-time service performance of edge devices.
Based on the system workflow described above, the long-term average inference accuracy and latency
can be respectively derived as
\begin{align}
	o_1(\boldsymbol{\alpha},\boldsymbol{p}, \boldsymbol{f}) &= \frac{1}{TN}\sum_{t=1}^{T} \sum_{n=1}^{N}(1-\alpha^t_n) Acc^t_n, \label{o1} \\
	o_2(\boldsymbol{\alpha},\boldsymbol{p}, \boldsymbol{f}) &= \frac{1}{TN}\sum_{t=1}^{T} \sum_{n=1}^{N}(1-\alpha^t_n)L^t_{n,\text{tot}}, \label{o2}
\end{align}
where $\boldsymbol{\alpha} = \{ \alpha_n^t,\ \forall n \in \mathcal{N},\ \forall t\in \mathcal{T}\}$,
$\boldsymbol{p} = \{ p_n^t, \ \forall n \in \mathcal{N},\ \forall t\in \mathcal{T}\}$, and
$\boldsymbol{f} = \{ f_n^t, \ \forall n \in \mathcal{N},\ \forall t\in \mathcal{T}\}$
represent the mode selection and resource allocation variables for all edge devices, respectively.
In contrast, energy consumption is a key metric from the perspective of edge devices themselves, particularly under constrained power budgets in practical deployments.
It accounts for the energy consumed during both training and inference modes, as well as the additional overhead $E_{n,\text{sw}}$ incurred when switching between these two modes. 
Accordingly, the \textcolor{blue}{average} energy consumption can be formulated as
\begin{align}\label{o3}
o_3(\boldsymbol{\alpha},\boldsymbol{p}, \boldsymbol{f})
&= \frac{1}{TN} \sum_{t=1}^{T} \sum_{n=1}^{N} \big[ \alpha_n^t E^t_{n,\text{tr}}
+ (1-\alpha^t_n) E^t_{n,\text{in}} \nonumber \\
&\quad + |\alpha^t_n - \alpha^{t-1}_n| E_{n,\text{sw}} \big].
\end{align}

Considering the interrelation among these three metrics, we treat each of them as an optimization objective and formulate the MOO problem accordingly as
\begin{subequations}\label{opt}
\begin{align}
    \mathcal{P}_1: \  &\max_{\{\boldsymbol{\alpha},\boldsymbol{p}, \boldsymbol{f}\}} \ (o_1,\ -o_2,\ -o_3), \label{obj} \\
    \text{s.t.} \quad
    & L^t \le L_{\max},\  \forall t\in\mathcal{T},\label{c1}\\[-2pt]
    & Q_n^t \le Q_{\max},\ \forall n\in\mathcal{N},\ \forall t\in\mathcal{T}, \label{c2}\\[-2pt]
    & D_n^t \le D_{\max},\ \forall n\in\mathcal{N},\ \forall t\in\mathcal{T}, \label{c3}\\[-2pt]
    & D_n^t > 0 ,\ \forall n\in \mathcal{M}^t,\ \forall t\in\mathcal{T},\label{c4}\\[-2pt]
    & L_{n,\text{in}}^t \le L^t,\ \forall n\in\mathcal{N} \setminus \mathcal{M}^t,\ \forall t\in\mathcal{T},\label{c5}\\[-2pt]
    & \alpha^t_n\in\{0,1\},\ \forall n\in\mathcal{N},\ \forall t\in\mathcal{T},\label{c6}\\[-2pt]
    & p_{\min}\le p^t_n\le p_{\max}, \ \forall n\in\mathcal{N},\ \forall t\in\mathcal{T},\label{c7}\\[-2pt]
    & f_{\min}\le f^t_n\le f_{\max}, \ \forall n\in\mathcal{N},\ \forall t\in\mathcal{T}, \label{c8}
\end{align}
\end{subequations}
where constraint~\eqref{c1} guarantees that the duration of each round does not exceed a predefined threshold, maintaining the stability of the FEEL system.
Constraints~\eqref{c2} and~\eqref{c3} limit the sizes of the request queue and dataset, respectively, preventing buffer overflows and maintaining storage limits of edge devices.
Constraint~\eqref{c4} restricts that any edge device participating in federated training holds a non-empty dataset.
Constraint~\eqref{c5} prevents the unnecessary waiting time through limiting the inference latency. 
Constraint~\eqref{c6} enforces the discrete nature of the mode selection variable.
Constraints~\eqref{c7} and~\eqref{c8} define the bounds for the allocated transmission power and CPU frequency of each edge device, respectively.

Obviously, solving this formulated MOO problem directly is very challenging because of several complexities.
\textit{First}, the discrete decision variables (i.e., $\boldsymbol{\alpha}$) make the problem a mixed-integer nonlinear programming formulation, which is known to be NP-hard.
\textit{Second}, the optimization objectives inherently exhibit conflicting relationships. For example, participating more frequently in federated training helps reduce the AoM, leading to higher inference accuracy; however, it may also increase energy consumption and prolong inference latency.
\textit{Third}, the considered system operates under online and dynamic conditions due to continuously arriving inference requests and fluctuating request patterns, while the lack of future information further complicates the problem in searching for long-term optimal strategies.

Given the above challenges,
in the next section, we propose a novel MODRL framework to efficiently address the formulated MOO problem.

\section{MODRL Algorithm}\label{algo}
In this section, we first model the formulated MOO problem $\mathcal{P}_1$ as a MOMDP. Then, adopting a novel constrained optimization perspective for MODRL, we design a constrained multi-objective proximal policy optimization (C-MOPPO) algorithm, which efficiently explores the Pareto front in the objective space. Finally, we provide the theoretical analysis for the C-MOPPO algorithm.

\subsection{MOMDP Formulation}
Due to the sequential dependence of the FEEL system state evolution, the formulated problem $\mathcal{P}_1$ can be naturally modeled as a MOMDP.
The MOMDP is represented by a tuple
$\langle \mathcal{S}, \mathcal{A}, \mathcal{P}, \boldsymbol{\mathcal{R}}, \gamma \rangle$, where
in each timestep $t$, an agent under current state $\boldsymbol{s}^t\in \mathcal{S}$ takes an action $\boldsymbol{a}^t\in \mathcal{A}$, and will transitions into a new state $\boldsymbol{s}^{t+1}$ with probability $Prob(\boldsymbol{s}^{t+1}|\boldsymbol{s}^t,\boldsymbol{a}^t) \in \mathcal{P}$.
Unlike the scalar reward in MDPs, MOMDP considers multiple objectives, so that the reward $\boldsymbol{r}^t = (r^{t}_{1}, r^{t}_{2}, \dots, r^{t}_{j}) \in   \boldsymbol{\mathcal{R}}$ is a $j$-dimensional vector representing rewards for each objective.
The discount factor $\gamma \in (0, 1]$ is used to balance immediate and future rewards.
For our considered system, the detailed definitions of the state space, action space, and reward function are provided as follows.

$\bullet$ \textit{State Space.}
\textcolor{blue}{The state represents the system information observed by the agent at each timestep, which is defined as
\begin{equation}
	\boldsymbol{s}^t = [\boldsymbol{a}^{t-1}, \boldsymbol{Q}^{t}, \boldsymbol{D}^{t}, L^{t-1}].
\end{equation}
}
Here, 
$\boldsymbol{a}^{t-1} = [\boldsymbol{a}^{t-1}_{n}]_{n\in \mathcal{N}}$
denotes the actions taken by all edge devices in the previous round, which will be introduced in Eq.~\eqref{action} later.
$\boldsymbol{Q}^{t} = [Q^t_1, Q^t_2, \dots, Q^t_N]$ and $\boldsymbol{D}^{t} = [D^t_1, D^t_2, \dots, D^t_N]$ represent the current request queue length and the size of training data for all edge devices.
\textcolor{blue}{$L^{t-1}$ is the duration of the previous round.}

$\bullet$ \textit{Action Space.}
For each edge device $n$, the agent determines its control decisions through an action vector that includes mode selection and resource allocation. Specifically, the action in round $t$ consists of the following three components:
$\boldsymbol{\alpha}^t = [\alpha^t_1, \alpha^t_2, \dots, \alpha^t_N]$ for mode selection (i.e., training or inference),
$\boldsymbol{p}^t = [p^t_1, p^t_2, \dots, p^t_N]$ for transmission power allocation, and
$\boldsymbol{f}^t = [f^t_1, f^t_2, \dots, f^t_N]$ for CPU frequency allocation.
Note that each $\alpha_n^t \in \{0,1\}$ is a discrete binary variable. The variables $p_n^t \in [0,1]$ and $f_n^t \in [0,1]$ are continuous values representing the normalized transmission power and CPU frequency, respectively.
Accordingly, the overall action in round $t$ is denoted by
\begin{equation}\label{action}
	\boldsymbol{a}^t = [\boldsymbol{\alpha}^t, \boldsymbol{p}^t, \boldsymbol{f}^t].
\end{equation}

$\bullet$ \textit{Reward Function.}
The instantaneous reward vector $\boldsymbol{r}^t$ in our MOMDP under a given state-action pair is defined as
\begin{equation}
	\boldsymbol{r}^t =
	\begin{cases}
		[r^t_1, r^t_2, r^t_3],\ \text{if all constraints are satisfied}, \\
		[\rho_1, \rho_2, \rho_3],\ \text{otherwise},
	\end{cases}
\end{equation}
where $\rho_1, \ \rho_2,\ \rho_3$ are penalty constants when the selected actions violate the constraints.
Here, the individual components of the reward vector are defined as
\begin{align}
	r^t_1 &= \frac{1}{N} \sum_{n=1}^{N} (1 - \alpha_n^t) Acc_n^t, \\
	r^t_2 &= -\frac{1}{N} \sum_{n=1}^{N} (1 - \alpha_n^t) L_{n,\text{tot}}^t, \\
	r^t_3 &= -\frac{1}{N} \sum_{n=1}^{N} \big[ \alpha_n^t E^t_{n,\text{tr}}
	+ (1-\alpha^t_n) E^t_{n,\text{in}} \nonumber \\
	&\quad + |\alpha^t_n - \alpha^{t-1}_n| E_{n,\text{sw}} \big],
\end{align}
which directly correspond to the optimization objectives defined in Eqs.~\eqref{o1}--\eqref{o3}.

After formulating the MOMDP, we introduce the following definitions to facilitate the comparison of policies in the MOO problem.
\begin{definition}
	(\textit{Pareto optimality})
	Given two policies $\pi$ and $\pi'$, policy $\pi$ \textit{\textbf{dominates}} $\pi'$ if and only if $\pi$ is no worse than $\pi'$ in all objectives and strictly better in at least one.
	A policy $\pi$ is \textit{\textbf{Pareto-optimal}} if it is not dominated by any other feasible policy.
	The set of all Pareto-optimal policies is called the \textbf{\textit{Pareto set}}, and the corresponding set of objective (cumulative reward) vectors forms the \textbf{\textit{Pareto front}}.
\end{definition}


\subsection{C-MOPPO Algorithm}
As illustrated in Fig.~\ref{alg}, the proposed C-MOPPO framework comprises three main stages.
In the first stage, called Pareto initialization, a set of scalarized reward tasks is defined and solved using the PPO algorithm~\cite{SchulmanPPO}. This process generates several initial policies, forming the initial solution set.
Following this, the algorithm alternates between the remaining two stages: policy selection and Pareto extension.
In the policy selection stage, a subset of promising policies is chosen.
In the Pareto extension stage, these selected policies are further optimized through constrained policy optimization, driving them toward unexplored regions of the Pareto front.
The detailed procedure is described below.
\subsubsection{Pareto Initialization}
The Pareto initialization stage aims to generate a set of diverse initial solutions by solving multiple learning tasks.
To achieve this, we first construct a set of preference vectors $\Omega$, in which each vector $\boldsymbol{\omega}_i \in \mathbb{R}^{(3)}$ represents a specific trade-off among the three objectives and satisfies $\sum_{j=1}^3 \omega_{i,j} = 1$, $\omega_{i,j} >0$.
These vectors serve as weights to transform the original multiple-dimensional reward into the scalar return via the preference function $f_{\boldsymbol{\omega}_i}(\boldsymbol{r}) = \boldsymbol{\omega}^{\top}_i\boldsymbol{r}$.
Each preference vector defines a learning task, where the goal is to learn a policy that maximizes the corresponding weighted sum of rewards.
Due to its stability and efficiency in policy optimization, we apply the PPO algorithm, a policy gradient method, 
to optimize these learning tasks.
\begin{figure}[!t]
	\centering
	\includegraphics[width=1\linewidth]{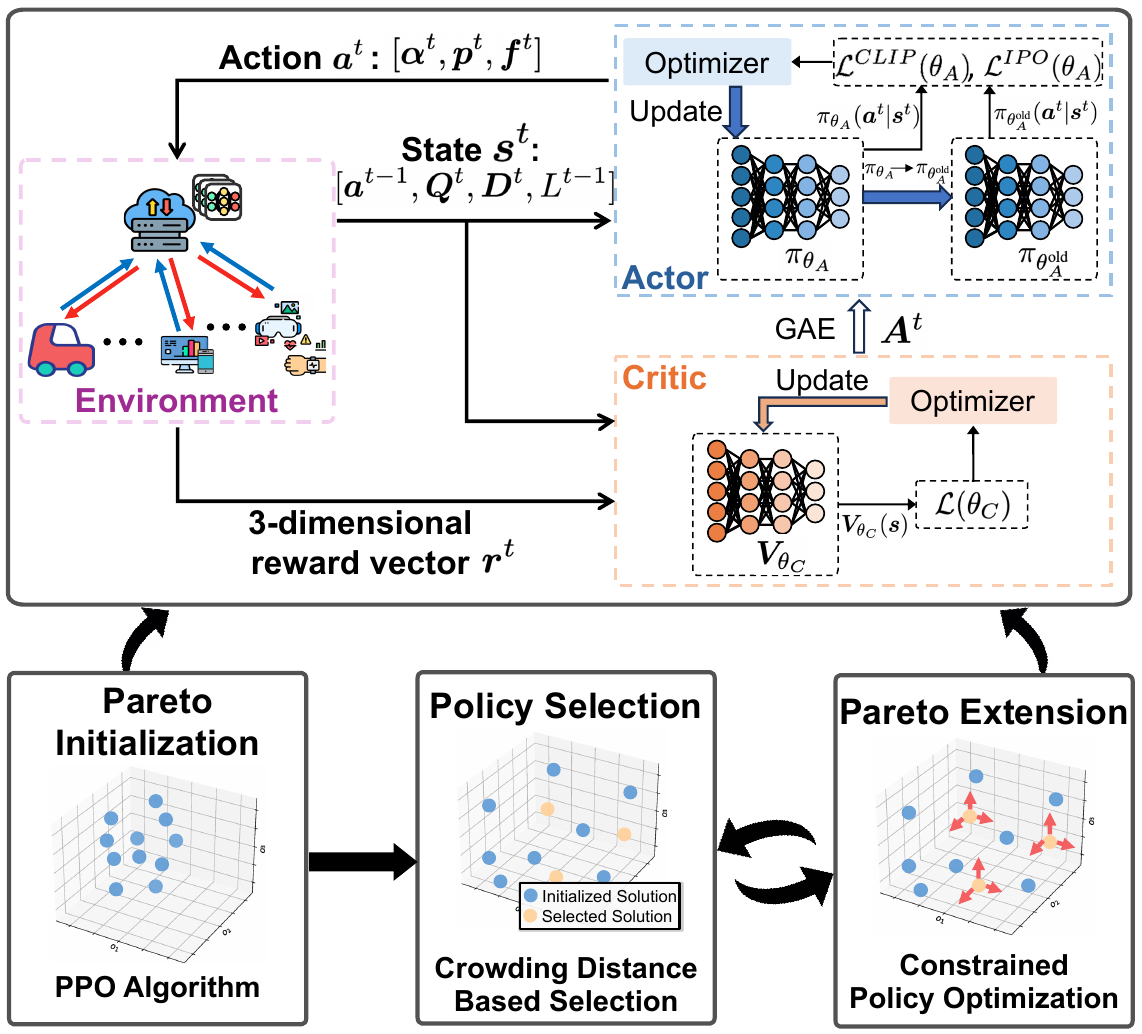}
	\caption{Overview of C-MOPPO.}
	\label{alg}
\end{figure}

Our designed PPO algorithm consists of two types of neural networks: the actor network, parameterized by $\theta_A$, defines the policy $\pi_{\theta_A}(\boldsymbol{a}|\boldsymbol{s})$;
and the critic network, parameterized by $\theta_C$, estimates the value function to guide the policy updates.
To ensure stable learning, 
a surrogate objective function is introduced that constrains policy updates by clipping the probability ratio between the new and old policies.
The clipped objective function is defined as
\begin{equation}\label{aupdate}
	\begin{aligned}
		\mathcal{L}^{CLIP}(\theta_A) =  \mathbb{E}_t\big[&\min  \big(r^t_{\text{ppo}}(\theta_A) A^t_{\boldsymbol{\omega}_i}, \\
		& \text{clip} (r^t_{\text{ppo}}(\theta_A) , 1 - \epsilon, 1 + \epsilon) A^t_{\boldsymbol{\omega}_i}\big)\big],
	\end{aligned}
\end{equation}
where $\epsilon$ is a predefined clipping parameter,
and $r^t_{\text{ppo}}(\theta_A) = \pi_{\theta_A}(\boldsymbol{a}^t|\boldsymbol{s}^t) / \pi_{\theta_{A}^{\text{old}}}(\boldsymbol{a}^t|\boldsymbol{s}^t)$ is the probability ratio between the new and old policies.
Note that the advantage value term $A^t_{\boldsymbol{\omega}_i}$ represents a scalarized estimation obtained by weighting the vectorized advantage function $\boldsymbol{A}^t$ with the preference vector $\boldsymbol{\omega}_i$, i.e., $A^t_{\boldsymbol{\omega}_i} = \boldsymbol{\omega}_i^{\top} \boldsymbol{A}^t$.
The vectorized advantage $\boldsymbol{A}^t$ is computed across all objectives using generalized advantage estimation (GAE)~\cite{SchulmanGAE}, a widely used technique that stabilizes training by balancing bias and variance through weighted summation of temporal difference (TD) errors. To this end, the actor network is optimized via gradient ascent to maximize the clipped objective function $\mathcal{L}^{CLIP}(\theta_A)$.

For the critic network, we extend the standard scalar state-value function to a vectorized state-value function, denoted by $\boldsymbol{V}_{\theta_C}(\boldsymbol{s})$.
This state-value function maps each state $\boldsymbol{s}$ to a vector of expected cumulative rewards for all objectives under the current policy $\pi$.
Based on this, the network parameters $\theta_C$ are updated by minimizing the following mean-squared error loss
\begin{equation}\label{cupdate}
	\mathcal{L}(\theta_C) = \mathbb{E}_t\big[||\boldsymbol{V}_{\theta_C}(\boldsymbol{s}^t) - \hat{\boldsymbol{V}}(\boldsymbol{s}^t)||^2  \big],
\end{equation}
where $\hat{\boldsymbol{V}}(\boldsymbol{s}^t) = \boldsymbol{r}^t + \gamma\boldsymbol{V}_{\theta_C}(\boldsymbol{s}^{t+1}) $ is the target value function.
This loss is minimized via gradient descent, using backpropagation to update $\theta_C$.

The primary steps of the PPO algorithm used in this stage are summarized in Algorithm~\ref{AlgPPO}.
Note that the resulting set of policies forms the initial solution set $\Psi_{\text{init}}$, which serves as the starting point for the subsequent stages.
\begin{algorithm}[!t]
	\caption{Proximal Policy Optimization Algorithm for Pareto Initialization}
	\label{AlgPPO}
	\KwIn{Preference vector set $\Omega$, actor learning rate $\xi_A$, critic learning rate $\xi_C$, clipping parameter $\epsilon$}
	\KwOut{Initial solution set $\Psi_{\text{init}}$}
	\For{each preference vector $\boldsymbol{\omega}_i \in \Omega$}{
		Initialize actor network with parameters $\theta_A$, critic network with parameters $\theta_C$\;
		\For{each episode}{
			Initialize state $\boldsymbol{s}^0$\;
			\For{$t = 1, \dots, T$}{
				Observe state $\boldsymbol{s}^t$, select and execute action $\boldsymbol{a}^t$ via policy $\pi_{\theta_A}$, then receive reward $\boldsymbol{r}^t$ and observe next state $\boldsymbol{s}^{t+1}$\;
				Store transition $(\boldsymbol{s}^t, \boldsymbol{a}^t, \boldsymbol{r}^t, \boldsymbol{s}^{t+1})$ into a trajectory list\;
			}
			\For{each epoch}{
				Sample a mini-batch of transitions from the trajectory list\;
				Update actor parameters $\theta_A$ by maximizing the clipped objective function Eq.~\eqref{aupdate} using gradient ascent with learning rate $\xi_A$\;
				Update critic parameters $\theta_C$ by minimizing the mean-squared error Eq.~\eqref{cupdate} using gradient descent with learning rate $\xi_C$\;
			}
		}
		Store the learned policy $\pi_{\theta_A}$ into $\Psi_{\text{init}}$\;
	}
\end{algorithm}

\subsubsection{Policy Selection}
Directly extending all policies in $\Psi_{\text{init}}$ is inefficient, as the initial policies are often unevenly distributed along the Pareto front, with some even lying off the true Pareto front, limiting the effectiveness of discovering new Pareto-optimal solutions.
Therefore, to effectively initiate the subsequent process, it is essential to first select a suitable subset of policies from the initial solution set $\Psi_{\text{init}}$.
We incorporate a policy selection mechanism prior to the extension process, where crowding distance~\cite{NSGA-II} is employed to identify policies with higher potential for Pareto extension.

Given a Pareto front approximation set $\Psi$ in the objective space, we first sort the policies in $\Psi$ according to the normalized value of each objective $i$, denoted by the sorted list $\tilde{Z}_i$.
For a specific policy $\psi \in \Psi $, supposing its index in the $\tilde{Z}_i$ is $l$, then the crowding distance of $\psi$, serving as an estimation of the density of policies around it, can be calculated as
\begin{equation}\label{CD}
	CD(\psi) = \sum_{i=1}^{3} \frac{\tilde{Z}_i(l+1) - \tilde{Z}_i(l-1)}{\tilde{Z}_{i,\max} - \tilde{Z}_{i,\min}},
\end{equation}
where $\tilde{Z}_{i,\max}$ and $\tilde{Z}_{i,\min}$ are the maximum and minimum values of the $i$-th objective in $\Psi$.
After computing the crowding distances of all current policies, we rank them in descending order based on their crowding distance values.
Intuitively, policies with larger crowding distances tend to be located in less-explored regions of the Pareto front, making them suitable for expanding the front toward new and diverse regions.
Therefore, the top $N_{\text{ext}}$ policies with the largest crowding distances are selected to form the candidate set $\Psi_{\text{ext}}$ for the subsequent Pareto extension process.
The detailed procedure of the policy selection is summarized in Algorithm~\ref{AlgSelection}.

\begin{algorithm}[!tb]
		\caption{Crowding Distance Based Policy Selection}\label{AlgSelection}
		\KwIn{Number of extension policies $N_{\text{ext}}$, current solution set $\Psi$}
		\KwOut{Extension solution set $\Psi_{\text{ext}}$}
		\For{each policy in $\Psi$}{
			Calculate the crowding distance by Eq.~\eqref{CD}\;
		}
		Sort the policies in $\Psi$ in descending order according to their crowding distances;\
		
		Initialize number of selected policies $N_{\text{sel}} = 0$\;
		\While{$N_{\text{sel}} < \min(N_{\text{ext}}, |\Psi|)$}{
			Add the $N_{\text{sel}}$-th policy in sorted $\Psi$ to $\Psi_{\text{ext}}$\;
			$N_{\text{sel}} = N_{\text{sel}} + 1$\;
		}
\end{algorithm}

\subsubsection{Pareto Extension}
Starting from the policies selected in the previous stage, the Pareto extension stage aims to explore new regions of the Pareto front by refining and optimizing these policies.
To achieve this goal, the constrained policy optimization approach~\cite{LiuR} can be utilized to guide the learning process in this stage.
Based on the principle of optimizing one primary objective while constraining the remaining objectives to acceptable levels, this constrained policy optimization approach is ideally suited for the Pareto extension stage.
Specifically, when focusing on maximizing a selected $l$-th reward in the original MOMDP, the framework requires the expected returns of the remaining objectives to exceed predefined thresholds.
The resulting constrained reinforcement learning formulation is given as 
\begin{equation}\label{CMDP}
	\mathcal{P}_2: \
	\max_{\pi} \ \mathcal{J}_{l}(\pi), \ \ \text{s.t.} \ \mathcal{J}_i(\pi)  \ge c_i, \  \forall i \in \{1,2,3\} \setminus \{l\}
	,
\end{equation}
where $\mathcal{J}_i(\pi) = \mathbb{E}_{\pi}\big[ \sum_{t}\gamma^tr^t_i \big]$ denotes the expected cumulative reward of the $i$-th objective under policy $\pi$, and $c_i$ represents the specified threshold for the $i$-th objective.

Note that the goal of the Pareto extension stage is to expand the current Pareto front approximation $\Psi_{\text{ext}}$ toward different objective directions by solving the constrained optimization problem in $\mathcal{P}_2$.
To effectively drive such expansion, it is crucial to ensure that the solutions obtained from $\mathcal{P}_2$ are themselves Pareto-optimal.
This highlights the importance of properly setting the constraint thresholds $c_i$ in $\mathcal{P}_2$.
To guide the selection of the constraint values, the following proposition establishes a sufficient condition.

\begin{proposition}\label{Prop1}
	Let $\tilde{Z}_i$ be the ascending sorted list of the $i$-th objective values from the Pareto front,
	and suppose that the index of the currently selected solution in this sorted list is $l$.
	If the constraint thresholds satisfy
	\begin{equation}
		c_i \ge \tilde{Z}_i(l-1), \ \forall i \in \{1,2,3\} \setminus \{l\},
	\end{equation}
	then the optimal solution of problem $\mathcal{P}_2$ is guaranteed to be a Pareto-optimal solution.
\end{proposition}

\begin{proof}
	Please see in Appendix B,~\cite{LiuR}.
\end{proof}

Proposition~\ref{Prop1} states that as long as the constraint thresholds for the non-optimized objectives are set no lower than the performance levels of nearby solutions on the current Pareto front approximation, the resulting solution from $\mathcal{P}_2$ will remain on the Pareto front.

However, applying Proposition~\ref{Prop1} in real-world settings may be challenging. 
It necessitates repeated evaluations and sorting of all candidate solutions at each optimization step, which can be computation-intensive and time-consuming.
Moreover, since Proposition~\ref{Prop1} provides only a sufficient condition, it imposes overly strict constraints, potentially excluding feasible policies that could contribute to a more complete Pareto front.
Inspired by the underlying intuition that trade-offs among objectives should be maintained to prevent excessive degradation of non-optimized objectives,
we introduce a more practical constraint specification method.
In this method, the strict requirement is relaxed by step-wise constraints that only ensure the expected returns do not fall significantly below those of the previous policy.
Specifically, the expected rewards of non-optimized objectives are required to remain above a fraction $\beta \in (0,1)$ of their previous values.
Accordingly, we reformulate the optimization $\mathcal{P}_2$ as follows.
\begin{equation}\label{CRL}
	\mathcal{P}_3:\
	\max_{\pi} \ \mathcal{J}_l(\pi), \ \
	\text{s.t.} \ \mathcal{J}_i(\pi) \geq \beta \mathcal{J}_i(\pi'), \ \forall i \in \{1,2,3\} \setminus \{l\},
\end{equation}
where $\mathcal{J}_i(\pi')$ is the expected reward of the $i$-th objective in the last optimization step.

To effectively solve problem $\mathcal{P}_3$, we apply the interior-point policy optimization (IPO) approach~\cite{LiuY}, which incorporates logarithmic barrier penalties into the objective function to enforce the constraints.
Define the logarithmic barrier function for each constraint in $\mathcal{P}_3$ as
\begin{equation}
	\phi(\mathcal{J}_i(\pi)) = \frac{\log(\mathcal{J}_i(\pi) - \beta \mathcal{J}_i(\pi'))}{\mu}.
\end{equation}
\textcolor{blue}{Here, the expected returns, including the negative latency and energy consumption as well as the positive accuracy metric, are already normalized during the policy selection stage. Therefore, the expected rewards $\mathcal{J}_i$ used for the barrier argument are strictly positive. This ensures the IPO framework remains mathematically rigorous and well-defined.}
\textcolor{blue}{The logarithmic barrier function acts as a penalty that increases as the policy approaches the feasibility boundary. Specifically, the barrier term tends to $-\infty$ as $\mathcal{J}_i(\pi) - \beta \mathcal{J}_i(\pi') \to 0^{+}$, preventing the constraint from being violated.}
\textcolor{blue}{Note that $\mu > 0$ is a hyperparameter that controls the strength of the logarithmic barrier, with larger values reducing its influence and yielding solutions of this optimization closer to the original problem.}
With this logarithmic barrier function,
the problem $\mathcal{P}_3$ can be converted into an unconstrained optimization problem as 
\begin{equation}
	\mathcal{P}_4: \
	\max_{\pi} \ \mathcal{J}_l(\pi) + \sum_{i=1, i\neq l}^{3}\phi(\mathcal{J}_i(\pi)).
\end{equation}
We integrate this reformulated problem $\mathcal{P}_4$ into the PPO framework for policy optimization.
For the $l$-th objective, we have the surrogate objective function as
\begin{equation}\label{IPO}
	\mathcal{L}^{IPO}(\theta_{A}) = \mathcal{L}^{CLIP}(\theta_A) + \sum_{i=1, i\neq l}^{3}\phi(\mathcal{J}_i(\pi_{\theta_A})).
\end{equation}
We then apply this augmented objective function to update policy parameters during the Pareto extension stage.
Using the IPO approach, the performance bound between $\mathcal{P}_3$ and $\mathcal{P}_4$ can be guaranteed, as established in Theorem~\ref{theo1}.
\begin{theorem}\label{theo1}
	The maximum gap between the optimal value of problem $\mathcal{P}_3$ and the objective value of $\mathcal{P}_4$ obtained via the IPO approach is bounded by $2/\mu$, where $\mu>0$ is the hyperparameter of the logarithmic barrier function, if the optimal policy is strictly feasible.
\end{theorem}

\begin{proof}
	According to standard analyses of interior-point method in~\cite{Boyd},  
	the maximum gap introduced by the logarithmic barrier reformulation in IPO approach is bounded by $m/\mu$,  
	where $m$ is the number of constraints in the original problem.  
	In our case, $\mathcal{P}_3$ contains $m=2$ constraints.  
	Therefore, the maximum gap does not exceed $2/\mu$.
\end{proof}


\begin{algorithm}[!t]
	\caption{C-MOPPO Algorithm}\label{CMOPPO}
	\KwIn{Number of Pareto extension steps $K$, number of constrained policy optimization steps $K'$, preference vector set $\Omega$}
	\KwOut{Solution set $\Psi$}
	\tcp{Pareto Initialization}
	\For{each preference vector $\boldsymbol{\omega}_i \in \Omega$}{
		Solve scalarized reward task using Alg.~\ref{AlgPPO} with $\boldsymbol{\omega}_i$\;
	}
	Obtain initial solution set $\Psi_{\text{init}}$ from all trained policies\;
	Set $\Psi \leftarrow \Psi_{\text{init}}$\;
	
	\tcp{Policy Selection}
	Apply Alg.~\ref{AlgSelection} to obtain extension set $\Psi_{\text{ext}} \subseteq \Psi$\;

	\tcp{Pareto Extension}
	\For{iteration $= 1,\dots,K / K'$}{
		\For{each policy $\pi \in \Psi_{\text{ext}}$}{
			Initialize policy $\pi_{0} = \pi$\;
			\For{$k = 1,\dots,K'$}{
				Update policy $\pi_k$ for each objective by maximizing Eq.~\eqref{IPO}\;
				Store policy $\pi_{k}$ into $\Psi$\;
			}
		}
		\tcp{Policy Selection}
		Obtain new $\Psi_{\text{ext}}$ by applying Alg.~\ref{AlgSelection} to~$\Psi$\;
	}
\end{algorithm}
To this end, the whole algorithm, referred to as C-MOPPO, is summarized in Algorithm~\ref{CMOPPO}.
\textcolor{blue}{For practical deployment, the C-MOPPO agent operates on the ES as a centralized scheduler. 
The ES performs instantaneous mode selection and resource allocation via lightweight neural network forward passes, while executing the policy fine-tuning asynchronously in the background.
With this deployment, transmitting the required state information (i.e., $Q^t_n$ and $D^t_n$) from each edge device incurs negligible communication overhead.}
\textcolor{blue}{
It is worth noting that 
the adoption of C-MOPPO is motivated by our formulated MOO problem.
Specifically, within the considered FEEL system, performance objectives exhibit asymmetric sensitivities, where further optimization of one objective may trigger disproportionate degradation in others.  
C-MOPPO addresses this via a constraint-aware Pareto extension that bounds relative performance degradation, acting as a dynamic safety net against severe policy fluctuations and ensuring the smooth, stable policy improvement in FEEL deployments.
}

\begin{proposition}
	The total time complexity of the proposed C-MOPPO algorithm is $O(	(n_{\text{epo}}|\Omega| + jKN_{\text{ext}})(\sum_{y=1}^{Y+1}M_{y-1} M_{y})
	) $,
	where $n_{\text{epo}}$ and $|\Omega|$ are the numbers of training epochs and learning tasks in the Pareto initialization stage, respectively. 
	$j$ is the number of objectives, 
	$K$ and $N_{\text{ext}}$ are the numbers of Pareto extension steps and extension policies, respectively.
	$Y$ represents the number of fully connected layers in the neural network, and $M_{y}$ is the number of neurons in the $y$-th layer.
\end{proposition}

\begin{proof}
	In Algorithm~\ref{CMOPPO}, the time complexity primarily comes from policy generation steps that involve neural network training (i.e., steps 2 and 11 in Algorithm~\ref{CMOPPO}).  
	Compared to these steps, other operations (e.g., steps 6 and 15 in Algorithm~\ref{CMOPPO}) are considered negligible.
	
	Note that for a neural network with $Y$ fully connected layers, where the $y$-th layer contains $M_{y}$ neurons, the time complexity of training is $O(\sum_{y=1}^{Y+1}M_{y-1} M_{y})$~\cite{LiJ}.
	As shown in Algorithm~\ref{AlgPPO}, PPO training in the Pareto initialization stage iteratively optimizes each learning task defined by the preference vector in $\Omega$ for $n_{\text{epo}}$ epochs. 
	Thus, the time complexity of Pareto initialization stage is $O(n_{\text{epo}}|\Omega| \sum_{y=1}^{Y+1}M_{y-1} M_{y})$.
	In the Pareto extension stage, the same network architecture are used but with different training times.
	Specifically, neural networks are updated $K$ times for $N_{\text{ext}}$ extension policies across $j$ objectives.
	Therefore, the time complexity of Pareto extension stage is $O(jKN_{\text{ext}} \sum_{y=1}^{Y+1}M_{y-1} M_{y})$.
	
	In summary, the total time complexity of the proposed C-MOPPO algorithm can be expressed as $O(
	(n_{\text{epo}}|\Omega| + jKN_{\text{ext}})(\sum_{y=1}^{Y+1}M_{y-1} M_{y})
	)$.
\end{proof}

\section{Performance Evaluation}\label{resu}
In this section, extensive experiments are conducted to evaluate the performance and advantages of the C-MOPPO algorithm.

\subsection{Simulation Settings}
We consider an FEEL system comprising a set of edge devices. 
Similar to~\cite{ChenX}, each device is equipped with a learning model whose update size is set to $G = 200$~KB. 
The CPU cycle requirements for training and inference on one data sample, denoted by $\chi_{n,\text{tr}}$ and $\chi_{n,\text{in}}$, are independently drawn from a uniform distribution $\mathcal{U}(10, 50)$.
Furthermore, to simulate heterogeneous request-training sample conversion, the mapping function $\mathscr{F}_{n}(\cdot)$ is configured such that half of the edge devices follow a one-to-one mapping between inference requests and training samples, while the remaining devices are evenly split between a many-to-one and a one-to-many mapping scheme.
The neural networks of the proposed algorithms are implemented using PyTorch~1.8.
Both the actor and critic networks in the PPO architecture consist of four fully connected layers, with each layer comprising 128 neurons.
\textcolor{blue}{The penalty constants are set to $[\rho_1, \rho_2, \rho_3] = [-0.1, -8, -50]$, which are strictly lower than the feasible instantaneous rewards, so that any constraint violation will trigger a sharp negative signal, guiding the policy toward the feasible region.}
\textcolor{blue}{
    The total training budget for the C-MOPPO algorithm is set to $1 \times 10^4$, including $6 \times 10^3$ training steps for the Pareto initialization and $4 \times 10^3$ training steps for the Pareto extension.}
Table~\ref{params} presents the remaining parameters and the detailed hyperparameters employed in the algorithm, most of which are commonly used in the existing literature~\cite{ChenX, SunY, ChenZ}.
\textcolor{blue}{All simulations are conducted on a computing platform equipped with a 16-core Intel Xeon CPU, 48GB RAM, and an NVIDIA GeForce RTX 3090 GPU rented from Vast.ai\footnote{https://vast.ai/.}.
For practical edge deployment, we consider a heterogeneous testbed consisting of NVIDIA Jetson TX2, Xavier NX, and Orin NX nodes.}
\textcolor{blue}{The policies generated by our proposed algorithm have been verified on these onboard GPUs.}

\begin{table}[!tbh]
	\begin{center}
		\caption{Simulation Parameters}
		\label{params}
		\resizebox{\linewidth}{!}{
			\begin{tabular}{c|c|c|c} 
				\toprule 
				\textbf{Parameter} & \textbf{Value} & \textbf{Parameter} & \textbf{Value}\\
				\midrule 
				Total rounds $T$ & $100$ & Duration threshold $L_{\max}$ & $5$ sec \\
				Bandwidth $B$ & $1$ MHz & Effective capacitance $e_n$ & $10^{-27}$ \\ 
				Discount factor $\gamma$ & $0.99$ & Clipping parameter $\epsilon$ & $0.2$\\
				Actor learning rate $\xi_A$ & $10^{-5}$ & Critic learning rate $\xi_C$ & $10^{-5}$ \\ 
				Log barrier parameter $\mu$ & $80$ & Constrained optimization steps $K'$ & $100$\\
                \textcolor{blue}{Preference set size $|\Omega|$} & \textcolor{blue}{$10$} & \textcolor{blue}{Random seeds} & \textcolor{blue}{$10$} \\
				\bottomrule 
			\end{tabular}
		}
	\end{center}
\vspace{-0.2in}
\end{table}

\begin{figure*}[!tbh]
	\centering
	\includegraphics[width=1\linewidth]{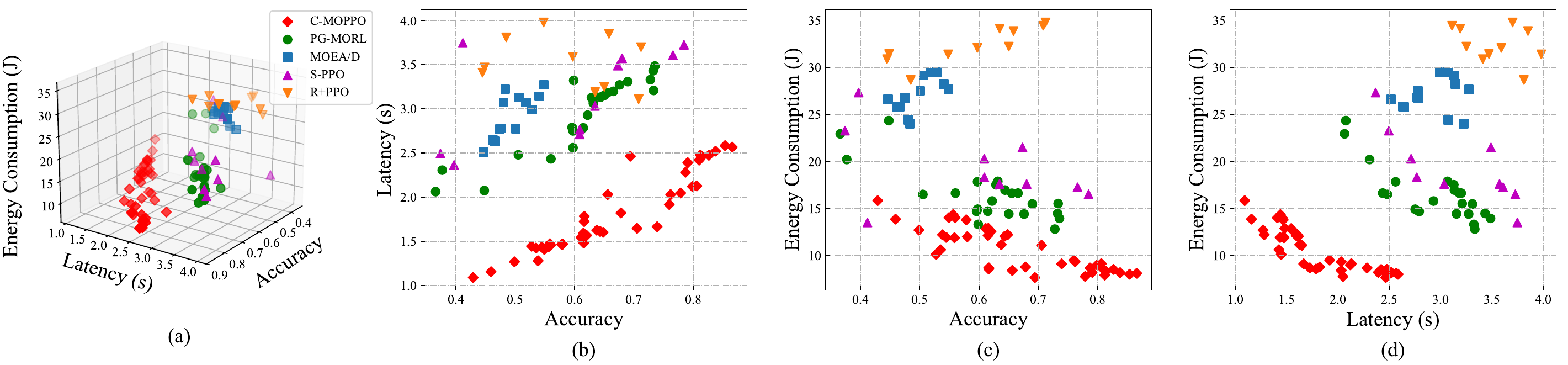}
	\caption{Pareto fronts obtained by C-MOPPO, PG-MORL, MOEA/D, S-PPO, and R+PPO algorithms. 
	}
	\label{PF}
	\vspace{-0.1in}
\end{figure*}

\subsection{Performance Indicators and Baselines}
To measure the quality of the approximated Pareto set, we consider two widely used performance indicators: hypervolume (HV) and sparsity (SP).
The HV is a comprehensive metric reflecting the convergence, spread, and homogeneity of the Pareto set.
The SP metric reflects the density of the Pareto set.
Specifically, given a Pareto approximation set $\Psi \subset \mathbb{R}^{(3)}$ and a reference point $\boldsymbol{z}_{\text{ref}} \in \mathbb{R}^{(3)}$, the HV and SP metrics are respectively calculated as  
\begin{equation}
	HV_{\boldsymbol{z}_{\text{ref}}}(\Psi) = \Lambda\big( \cup_{\boldsymbol{z} \in \Psi} \{ \boldsymbol{z}' \in \mathbb{R}^{(3)} \mid \boldsymbol{z} \succ \boldsymbol{z}' \succ  \boldsymbol{z}_{\text{ref}} \} \big),
\end{equation}
\begin{equation}
	SP(\Psi) = \frac{1}{|\Psi| -1} \sum_{i=1}^{3}\sum_{l=1}^{|\Psi|-1}\big( \tilde{Z}_i(l+1) - \tilde{Z}_i(l) \big)^2,
\end{equation}
where $\Lambda(\cdot)$ denotes the $3$-dimensional Lebesgue measure of a set, and $\tilde{Z}_i(l)$ is the $l$-th value in sorted list for the $i$-th objective values in $\Psi$.
Briefly, a higher HV value and a lower SP value indicate better performance of the Pareto approximation set.

To demonstrate the superiority of the C-MOPPO approach, several baseline algorithms are implemented for comparison, including the prediction-guided multi-objective reinforcement learning algorithm (PG-MORL)~\cite{XuJ} and the multi-objective evolutionary algorithm based on decomposition (MOEA/D)~\cite{moea}. 
Additionally, to validate our system designs, we implement two single-objective optimization (SOO) baselines: scalarized-PPO (S-PPO) and PPO with random mode selection (R+PPO).

$\bullet$ {PG-MORL}: This algorithm employs a prediction-guided evolutionary reinforcement learning framework, where in each generation, an analytical model is used to estimate the expected improvement of each policy along multiple objective directions, followed by solving an optimization problem to select the most promising policies and directions for Pareto front improvement.

$\bullet$ {MOEA/D}: This algorithm decomposes the multi-objective problem into multiple scalar subproblems, each of which is optimized using information from its neighboring subproblems to lower the computational complexity.

$\bullet$ {S-PPO}: This approach converts the original MOO problem into multiple SOO problems using $10$ different well-designed weights, with each scalarized problem then solved separately using the PPO algorithm. 

$\bullet$ {R+PPO}: This method adopts a similar scalarization approach as S-PPO but separates mode selection from resource allocation. 
Specifically, edge devices randomly select their modes in each round, and the PPO algorithm is utilized for the communication and computation resource allocation.

\textcolor{blue}{Note that for fair comparison, we fixed the total training budget at $1 \times 10^4$ PPO training steps for C-MOPPO and all baseline algorithms. 
The training steps allocated to each weight of S-PPO and R+PPO are set to $1/10$ of the total steps used by the multi-objective
algorithms.}



\subsection{Performance Evaluations}
Fig.~\ref{PF} illustrates the Pareto fronts achieved by C-MOPPO and four baselines.
Note that while S-PPO and R+PPO are scalarization-based methods, their solutions with different weight combinations are plotted in the multi-objective space to serve as approximations of Pareto fronts for comparison.
Fig.~\ref{PF}(a) presents the three-dimensional distribution of the policies generated by different algorithms in the objective space, and Fig.~\ref{PF}(b)-(d) provides pairwise projections for a clearer view of the trade-offs between objectives.
\textcolor{blue}{These metrics represent the average values formulated in Eqs.~\eqref{o1}--\eqref{o3}.}
It can be observed that the C-MOPPO algorithm achieves the most comprehensive coverage of the Pareto front, generating a well-distributed set of solutions that deliver superior trade-offs across the three considered objectives.
The PG-MORL also demonstrates competitive performance due to its reinforcement learning foundation, but shows noticeable gaps with C-MOPPO, particularly in regions emphasizing the latency objective.
\textcolor{blue}{This gap primarily occurs because PG-MORL relies on a prediction model to estimate objective improvements. 
For highly complex and tightly coupled metrics such as inference latency in the FEEL system, the prediction tends to fall into local optima, degrading overall algorithm performance.
Unlike PG-MORL, C-MOPPO achieves superior performance by utilizing a rigorous constrained policy optimization approach in the Pareto extension, which theoretically bounds the performance degradation of other objectives and ensures mathematically stable exploration.}

\textcolor{blue}{
MOEA/D and S-PPO exhibit significantly inferior performance to C-MOPPO.
Specifically, the performance gap of MOEA/D arises from its reliance on heuristic genetic operators, which require substantial computational effort to generate populations in the decision space and thus result in clustered solutions with poor Pareto front diversity. 
In contrast, C-MOPPO leverages the policy gradient framework of PPO with neural networks to capture complex state–action mappings, achieving broader and more uniform Pareto front coverage.
Although S-PPO exhibits a broader distribution than MOEA/D, its reliance on fixed scalarization inherently restricts it to specialized policies, leading to missed solutions in non-convex regions.
C-MOPPO overcomes this by alternating between the policy selection stage and the Pareto extension stage, driving the algorithm to iteratively expand into the sparse, unexplored regions of the Pareto front.}
Finally, R+PPO delivers the worst Pareto front, with significantly higher energy consumption and latency compared to C-MOPPO. This occurs because random mode selection may schedule edge devices to train when they have accumulated inference requests, and perform inference when fresh training data is available. This validates the importance of our joint optimization framework for guaranteeing the FEEL system performance.

Fig.~\ref{number} presents the impact of the number of edge devices on algorithm performance in terms of two metrics. 
For the calculation of HV value, we select the reference point $\boldsymbol{z}_{\text{ref}} = (0, 5, 44)$, which can bound the three optimization objectives across all experiments.
As shown in Fig.~\ref{number}(a), the HV values of all algorithms decrease as the number of edge devices increases, reflecting the increasing difficulty in searching trade-offs in a more complex decision space. 
Among all methods, our C-MOPPO consistently achieves the highest HV values, demonstrating superior scalability and robustness. 
PG-MORL ranks second in most cases but is surpassed by MOEA/D when the number of edge devices reaches $30$. This performance drop can be attributed to the expansion of the action space, which increases the learning difficulty for reinforcement learning-based methods. 
MOEA/D and S-PPO exhibit similar performance levels, both inferior to C-MOPPO and PG-MORL in most cases. This gap can be attributed to their inherent scalarization strategies, which limit effective exploration in complex, dynamic environments.
R+PPO consistently performs the worst in terms of the HV metric, indicating its limited capacity to explore Pareto-optimal solutions.
Fig.~\ref{number}(b) shows the corresponding SP values, where all methods exhibit increasing trends as the number of devices grows. The C-MOPPO maintains the lowest SP values across all settings.
PG-MORL initially shows competitive performance but is outperformed by the MOEA/D when the number of devices reaches $20$ or higher. This may result from the degradation of PG-MORL’s prediction model under high-dimensional action spaces, which prevents it from accurately identifying promising directions.
Both S-PPO and R+PPO exhibit higher SP values than the other methods, with S-PPO showing the worst performance. This reflects that approximating Pareto fronts using only a limited set of weighted single-objective problems is ineffective.
\begin{figure}[t]
	\centering
	\includegraphics[width=1.015\linewidth]{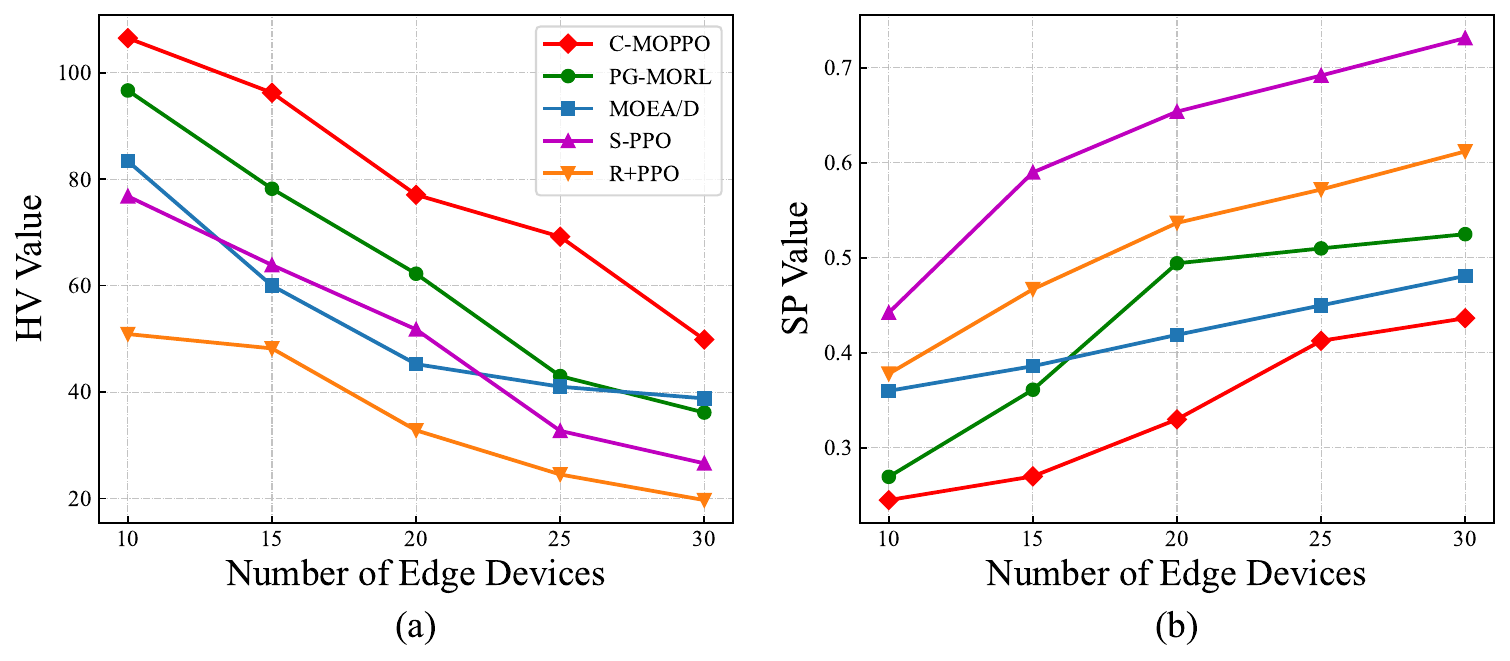}
	\caption{Performance comparison with different numbers of devices $N$ in terms of (a) HV Value; (B) SP Value.}
	\label{number}
\end{figure}

Fig.~\ref{lam} illustrates the algorithm performance in terms of HV and SP metrics as the average request arrival rate varies from $2$ to $10$.
As shown in Fig.~\ref{lam}(a), both C-MOPPO, PG-MORL, and S-PPO exhibit a non-monotonic trend in the HV metric, where performance initially improves, subsequently declines. 
This is because, at lower request arrival rates, increased request frequency can provide more training data, improving the inference accuracy. However, at higher arrival rates, the consequent increase in inference latency and energy consumption begins to dominate the objective space, resulting in reduced HV values.
Overall, the C-MOPPO algorithms demonstrate superior performance compared to the other baselines, achieving approximately $16\%$ improvement over the second-best PG-MORL algorithm on average across all scenarios.
After the two leading methods, S-PPO ranks next in HV performance, slightly ahead of MOEA/D, and R+PPO shows the poorest performance.
Regarding the SP metric, as depicted in Fig.~\ref{lam}(b), C-MOPPO maintains consistently low and stable SP values around $0.34$ across all request arrival rates, demonstrating its robust capability to generate well-distributed solutions regardless of workload intensity. 
The other methods, particularly S-PPO and MOEA/D, exhibit a clear upward trend from about $0.40$ to roughly $0.68$ as the request arrival rate increases, indicating the difficulty in maintaining dense solutions under heavier system loads for these algorithms.
\begin{figure}[t]
	\centering    \includegraphics[width=1.015\linewidth]{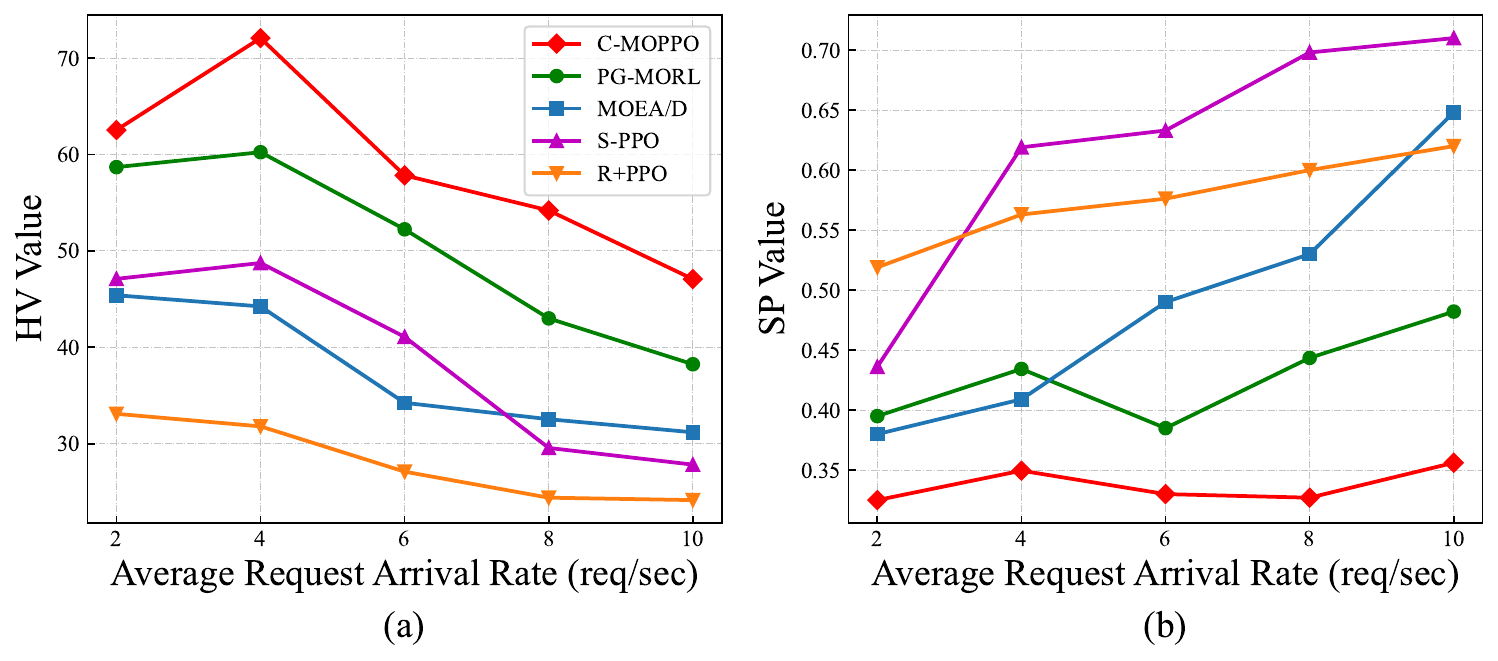}
	\caption{Performance comparison with different average request arrival rates $\bar{\lambda}$ in terms of (a) HV Value; (B) SP Value.}
	\label{lam}
\end{figure}


Fig.~\ref{convergence} examines the convergence behavior of our presented C-MOPPO algorithm compared to PG-MORL and MOEA/D algorithms. 
Since this analysis focuses on the convergence of the HV value for the evolving Pareto set, scalarization-based methods (i.e., S-PPO and R+PPO) are not included because they optimize independent weighted SOO problems rather than a unified Pareto set.
Overall, C-MOPPO demonstrates the fastest and most stable convergence, reaching the highest HV value of approximately $79$ and maintaining consistent performance thereafter, indicating that its solution set effectively converges the Pareto front.
PG-MORL achieves better performance than the MOEA/D method but still falls short of C-MOPPO, converging to around $66$, even though both algorithms start the second stage with a similar initial solution set. This gap can be attributed to its heavy reliance on the inherent prediction model, whose accuracy is insufficient for the complex optimization problem.
MOEA/D converges to substantially lower HV values of around $48$. This poor performance results from the inefficiency of evolutionary algorithms in high-dimensional parameter spaces, which is further exacerbated by the dynamic and uncertain nature of the FEEL system.
\begin{figure}[!t]
	\centering
    \includegraphics[width=0.7\linewidth]{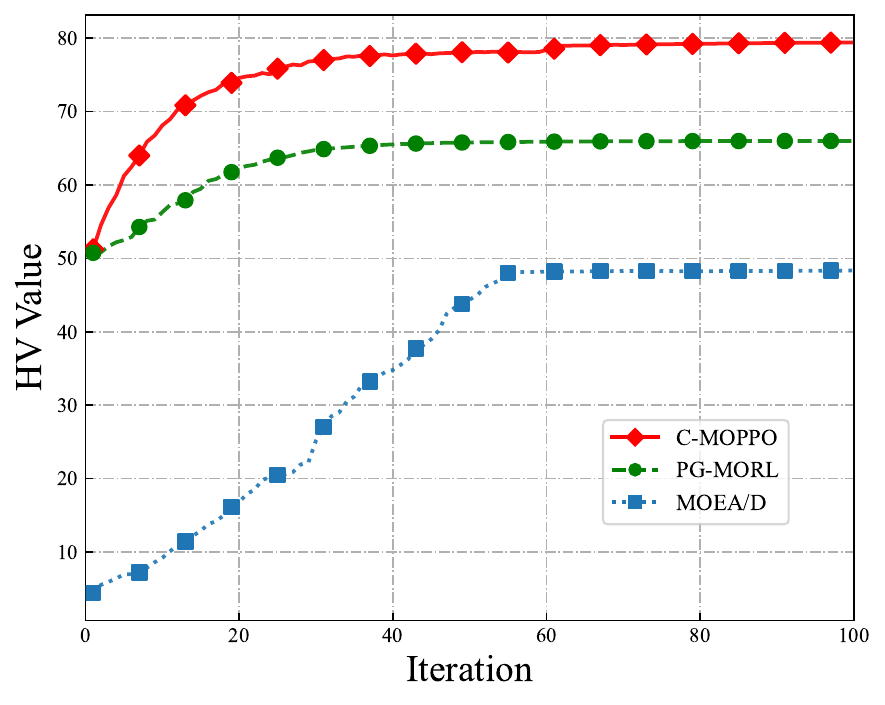}
	\caption{Convergence comparison of C-MOPPO, PG-MORL, and MOEA/D algorithms.}
	\label{convergence}
\end{figure}

\subsection{Sensitivity and Distribution Analysis}
Fig.~\ref{be} shows the effect of varying the return constraint hyperparameters $\beta$ on the quality of the Pareto-optimal solutions generated by the C-MOPPO algorithm.
When $\beta$ increases from $0.5$ to $0.9$, the HV value improves steadily from $56.6$ to $75.2$, while the SP value decreases from $0.56$ to $0.36$. This indicates that stronger constraints promote better overall performance, leading to higher-quality and more diverse solutions. This improvement is intuitive, as a larger $\beta$ effectively limits the degradation of secondary objectives when optimizing a primary one, thereby improving coverage of the Pareto front.


\textcolor{blue}{
Since inference latency is critical for practical QoS, we select the representative latency-oriented policy from each algorithm's Pareto set for further distribution analysis.
Fig.~\ref{cdf} presents the cumulative distribution function (CDF) of inference latency under the selected policies, collected across all rounds of the simulation.
As observed, the CDF curve of C-MOPPO exhibits a sharp ascent, rapidly reaching $1.0$ at a significantly lower latency bound. 
PG-MORL follows next with a moderate distribution, while MOEA/D and S-PPO exhibit similar but inferior trends, both revealing significant latency variations.
R+PPO demonstrates much flatter curves with severe long-tail effects, indicating the random mode selection frequently forces inference requests to wait in training mode, causing severe latency degradation.
By strategically switching modes, C-MOPPO avoids such delays and provides deterministic latency bounds, ensuring the strict QoS constraints.}

\begin{figure}[!t]
    \begin{minipage}{0.52\linewidth}
        \centering
        \includegraphics[width=\linewidth]{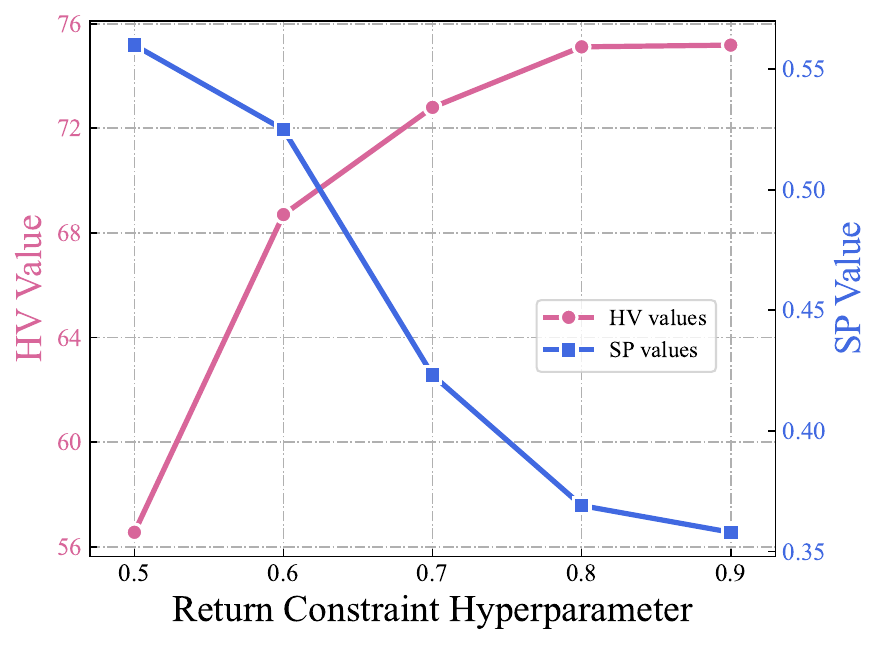}
        \caption{HV and SP values with different return constraint hyperparameters $\beta$.}
        \label{be}
    \end{minipage}
    \hfill 
    \begin{minipage}{0.465\linewidth}
        \centering
        \includegraphics[width=\linewidth]{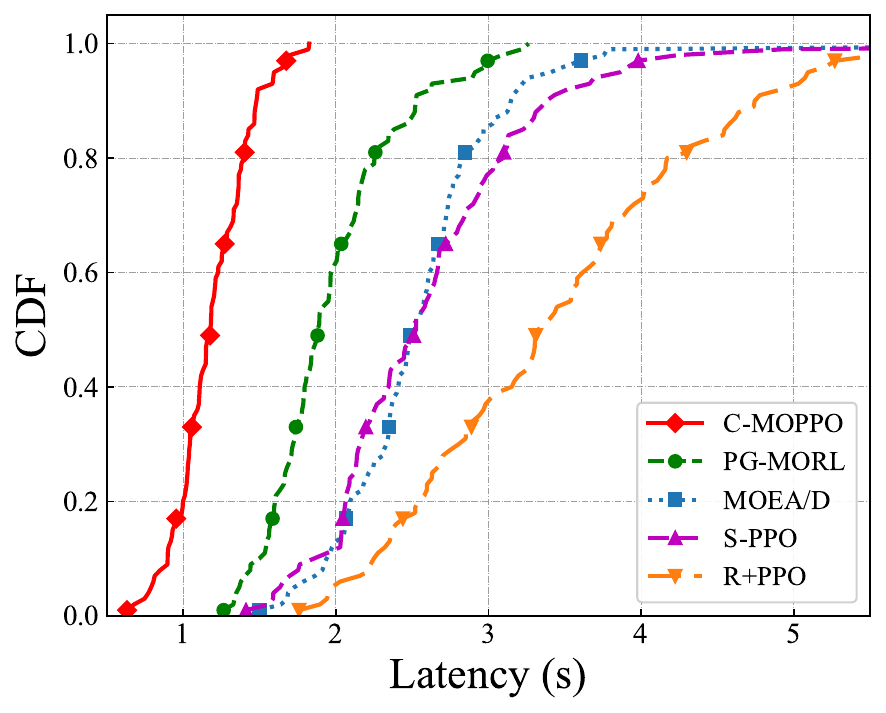}
        \caption{Comparison of latency distributions across different algorithms.}
        \label{cdf}
    \end{minipage}
\end{figure}

\textcolor{blue}{\subsection{Evaluation on Standard Classification Task}\label{real}
We conduct comprehensive evaluations using the real-world CIFAR-10 dataset~\cite{Krizhevsky}, which consists of $60,000$ color images across ten classes. This dataset was used to generate the inference task streams following the methodology in~\cite{WangH2}. A standard CNN with three convolutional blocks and a lightweight fully connected classifier was deployed across edge devices for FEEL training and inference.
\subsubsection{Validation of the Accuracy Modeling}
We evaluate the fidelity of our proposed accuracy mapping function by comparing it with a theoretical Oracle variant of C-MOPPO.
Specifically, during the training phase, the DRL agent in the original C-MOPPO is guided by the reward estimated via $\mathscr{G}(\cdot,\cdot)$. 
To demonstrate robustness, we instantiate $\mathscr{G}(\cdot,\cdot)$ into three diverse variants: the default logarithmic-exponential function defined in Section~\ref{acc}, a power-law-exponential function ( $\mathscr{G}(x,y) = \kappa_1 x^{\kappa_2}\cdot e^{-\kappa_3 y}$), and a logarithmic-rational function ($\mathscr{G}(x,y) = \kappa_1 \log (1+\kappa_2x)\cdot \frac{1}{1+\kappa_3y}$), where the parameters $\kappa_i$ are specifically set for each variant.
In contrast, the Oracle variant operates under a hypothetical assumption of perfect future knowledge so that it completely bypasses our proxy and relies directly on the true inference accuracy as its reward signal. Fig.~\ref{accV} illustrates the convergence behavior in terms of the HV value (capturing multi-objective trade-offs rather than isolated accuracy).
Crucially, the plotted HV values for all evaluated methods are calculated using the ground-truth inference accuracy of the trained CNN models on the CIFAR-10 dataset.
Several key observations can be drawn from the results. 
First, despite lacking access to the true inference accuracy during optimization, our C-MOPPO successfully converges to a Pareto front that exhibits a final performance gap of less than $5\%$ compared to the Oracle. 
Second, the negligible $1.31\%$ performance difference among the different variants of $\mathscr{G}(\cdot, \cdot)$ proves that the algorithm's superiority is highly robust to the specific mathematical structure of the accuracy mapping function. 
Finally, it is noteworthy that the mapping-guided C-MOPPO algorithms demonstrate faster and more stable convergence behavior than the Oracle. 
This verifies that our mapping function acts as an effective reward shaping mechanism, providing the DRL agent with smooth signals that accelerate policy learning without compromising the correct optimization direction.
}

\begin{figure}[!t]
	\centering    \includegraphics[width=0.7\linewidth]{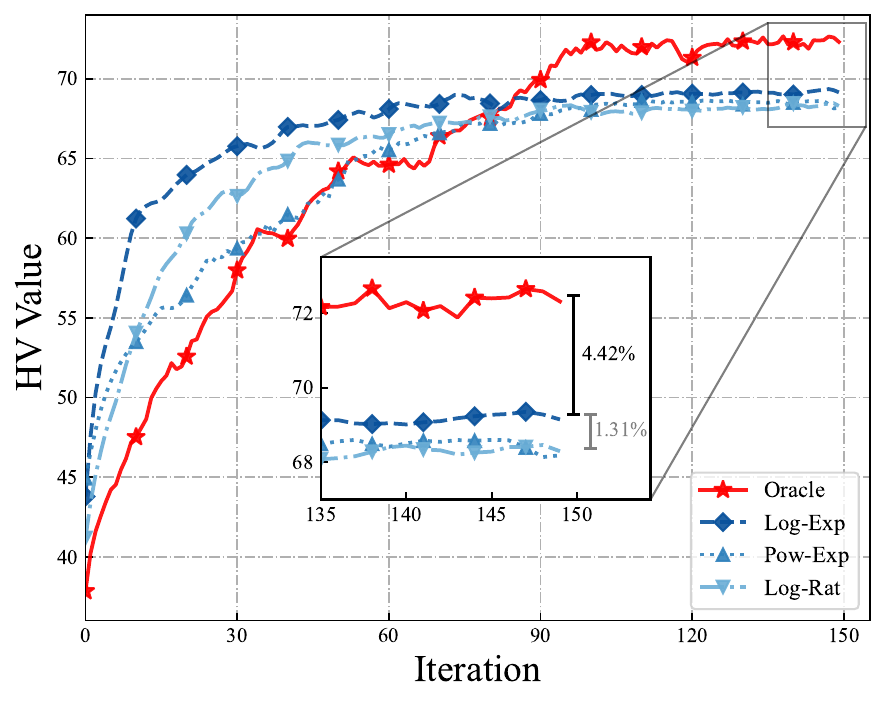}
	\caption{Convergence behavior of different accuracy mapping functions versus the Oracle.}
	\label{accV}
\end{figure}

\begin{table*}[!th]
\centering
\caption{Detailed Performance of Top-3 Preference-Oriented Policies across Different Algorithms}
\label{top3}
\begin{tabular}{c|ccc|ccc|ccc}
\toprule
\multirow{2}{*}{\textbf{Algorithms}} & \multicolumn{3}{c|}{\textbf{Accuracy-Oriented Policies}} & \multicolumn{3}{c|}{\textbf{Latency-Oriented Policies}} & \multicolumn{3}{c}{\textbf{Energy-Oriented Policies}} \\
\cmidrule{2-10}
 & \textbf{Acc. (\%)} & Lat. (s) & Energy (J) & Acc. (\%) & \textbf{Lat. (s)} & Energy (J) & Acc. (\%) & Lat. (s) & \textbf{Energy (J)} \\
\midrule

C-MOPPO 
 & \makecell{\textbf{89.1} \\ \textbf{88.6} \\ \textbf{88.2}} 
 & \makecell{3.61 \\ 3.82 \\ 3.50} 
 & \makecell{16.8 \\ 16.1 \\ 17.2} 
 & \makecell{72.9 \\ 75.1 \\ 75.7} 
 & \makecell{\textbf{2.68} \\ \textbf{2.75} \\ \textbf{2.84}} 
 & \makecell{18.8 \\ 16.4 \\ 17.0} 
 & \makecell{69.4 \\ 77.2 \\ 78.1} 
 & \makecell{3.52 \\ 3.04 \\ 3.47} 
 & \makecell{\textbf{14.8} \\ \textbf{14.9} \\ \textbf{15.2}} \\
\midrule
PG-MORL 
 & \makecell{87.2 \\ 85.8 \\ 84.3} 
 & \makecell{4.48 \\ 4.20 \\ 4.43} 
 & \makecell{20.5 \\ 23.2 \\ 22.6} 
 & \makecell{67.3 \\ 74.5 \\ 67.2} 
 & \makecell{3.67 \\ 3.69 \\ 3.92} 
 & \makecell{22.9 \\ 25.1 \\ 20.7} 
 & \makecell{65.2 \\ 69.8 \\ 74.2} 
 & \makecell{4.13 \\ 4.12 \\ 4.40} 
 & \makecell{19.1 \\ 19.6 \\ 20.2} \\
\midrule
MOEA/D 
 & \makecell{86.2 \\ 83.9 \\ 81.1} 
 & \makecell{4.47 \\ 4.10 \\ 3.99} 
 & \makecell{27.7 \\ 28.5 \\ 32.9} 
 & \makecell{66.5 \\ 74.5 \\ 76.6} 
 & \makecell{3.64 \\ 3.65 \\ 3.79} 
 & \makecell{31.7 \\ 29.1 \\ 28.7} 
 & \makecell{68.4 \\ 86.2 \\ 67.0} 
 & \makecell{4.22 \\ 4.47 \\ 4.07} 
 & \makecell{25.0 \\ 27.7 \\ 28.1} \\
\midrule
S-PPO 
 & \makecell{87.5 \\ 83.8 \\ 82.0} 
 & \makecell{4.78 \\ 4.61 \\ 4.57} 
 & \makecell{24.9 \\ 25.3 \\ 25.6} 
 & \makecell{69.1 \\ 67.2 \\ 70.6} 
 & \makecell{3.96 \\ 4.08 \\ 4.21} 
 & \makecell{30.8 \\ 26.5 \\ 23.9} 
 & \makecell{71.8 \\ 70.6 \\ 87.5} 
 & \makecell{4.74 \\ 4.21 \\ 4.78} 
 & \makecell{23.5 \\ 23.9 \\ 24.9} \\
\midrule
R+PPO 
 & \makecell{76.4 \\ 75.5 \\ 74.8} 
 & \makecell{4.72 \\ 4.11 \\ 4.85} 
 & \makecell{34.7 \\ 34.8\\ 33.8} 
 & \makecell{67.9 \\ 75.5 \\ 68.5} 
 & \makecell{3.98 \\ 4.11 \\ 4.51} 
 & \makecell{31.3 \\ 34.8\\ 34.4} 
 & \makecell{63.4 \\ 65.0 \\ 67.9} 
 & \makecell{4.59 \\ 4.84 \\ 3.98} 
 & \makecell{26.8 \\ 28.1 \\ 31.3} \\
\bottomrule
\end{tabular}
\end{table*}

\textcolor{blue}{\subsubsection{Detailed Performance Comparison} We detail the performance of the algorithms when prioritizing different system objectives in the implemented classification task. Table~\ref{top3} presents the top-3 preference-oriented policies generated by C-MOPPO and all baselines, evaluated in terms of empirical inference accuracy, latency, and energy consumption.
It can be observed that the proposed C-MOPPO algorithm consistently achieves the best performance across all three preference-oriented settings (highlighted in bold). Under the accuracy-oriented policies, C-MOPPO attains the highest inference accuracy of up to $89.1\%$, significantly outperforming the best baseline (PG-MORL at $87.2\%$). Similarly, when prioritizing latency or energy consumption, C-MOPPO successfully minimizes the respective metrics to $2.68$ s and $14.8$ J, demonstrating its robust capability to push the Pareto front further in any desired optimization direction.
Additionally, the results clearly illustrate the inherent trade-offs among the three objectives. For instance, achieving C-MOPPO's peak accuracy incurs higher latency and energy consumption.
This aligns with our system model: maintaining high accuracy requires frequent training, inevitably leading to higher inference latency and energy costs. Conversely, strictly minimizing latency or energy consumption leads to a corresponding drop in accuracy (e.g., dropping to $69.4\%$ at $14.8$ J).
Overall, the baseline algorithms underperform C-MOPPO. 
While PG-MORL and MOEA/D can find some sub-optimal trade-offs, they fall short of C-MOPPO across all primary metrics. 
A policy overlapping phenomenon is observed in the scalarization-based methods, where identical policies repeatedly emerge under different preference orientations. This redundancy exposes their limited exploration capabilities and their inability to discover a diverse, well-distributed Pareto set like C-MOPPO.
}


 \section{Conclusion}\label{concl}
This paper investigates the joint mode selection, communication, and computation resource allocation problem for edge devices in the FEEL system. 
To capture the practical characteristics of online environments, we designed an efficient conversion mechanism that bridges inference requests and training data for resource-constrained edge devices, 
and formulated a comprehensive accuracy metric that accounts for training data volume, freshness, and model staleness.
Building upon the multi-objective Markov decision process, we developed a constrained multi-objective proximal policy
optimization algorithm, termed C-MOPPO, maximizing inference accuracy while simultaneously minimizing latency and energy consumption. 
Extensive experimental results validate the effectiveness of the proposed C-MOPPO algorithm in generating Pareto-optimal solutions and demonstrate its substantial advantages over existing baselines.

\bibliographystyle{IEEEtran}
\bibliography{IEEEtranbib}

\end{document}